\documentclass{article}

\usepackage[final,nonatbib]{neurips_2020}
\usepackage[numbers,sort&compress]{natbib}

\usepackage[utf8]{inputenc} 
\usepackage[T1]{fontenc}    
\usepackage{hyperref}       
\usepackage{url}            
\usepackage{booktabs}       
\usepackage{amsfonts}       
\usepackage{nicefrac}       
\usepackage{microtype}      
\usepackage{amssymb}
\usepackage{subfig}
\usepackage{graphicx}
\usepackage{color}
\usepackage{complexity}
\usepackage{amsmath}
\usepackage{amsthm}
\usepackage{tabularx}
\usepackage{fancyvrb}
\usepackage{comment} 
\usepackage{algorithm}
\usepackage{algorithmic}
\usepackage{float}
\usepackage{wrapfig}
\usepackage{mathtools}
\usepackage{thmtools,thm-restate}
\usepackage{paralist} 
\usepackage{multirow}
\usepackage{scrextend}
\usepackage{enumitem}
\usepackage[noorphans,vskip=0ex]{quoting}
\usepackage{bbm}
\usepackage{xcolor}
\usepackage{xspace}
\usepackage{lscape}
\usepackage{hyperref}       
\usepackage{mathpazo}
\usepackage{times}
\usepackage{tikz}
\usepackage{collcell}
\usepackage{algorithm}
\usepackage{algorithmic}
\usepackage{wrapfig}

\theoremstyle{definition}
\newtheorem{definition}{Definition}[section]

\newtheorem{lemma}{Lemma}[section]

\renewcommand{\algorithmicrequire}{\textbf{Input:}}
\renewcommand{\algorithmicensure}{\textbf{Output:}}
\newcommand{\domain}{\mathcal{D}}

\newcommand{\data}{\mathcal{X}}
\newcommand{\hdata}{\mathcal{Z}}
\newcommand{\ydata}{\mathcal{Y}}
\newcommand{\FF}{\mathcal{F}}

\newcommand{\err}{\varepsilon}

\newcommand{\RR}{\mathbb{R}}

\newcommand{\Exp}{\mathbb{E}}
\newcommand{\HH}{\mathcal{H}}
\newcommand{\xx}{\mathbf{x}}
\newcommand{\ww}{\mathbf{w}}

\newcommand{\defeq}{\vcentcolon=}
\newcommand{\eqdef}{=\vcentcolon}
\newcommand{\eps}{\varepsilon}

\DeclarePairedDelimiterX{\inp}[2]{\langle}{\rangle}{#1, #2}

\newcommand{\jsd}{D_{\text{JS}}}
\newcommand{\kl}{D_{\text{KL}}}
\newcommand{\dist}{\mathcal{D}}
\newcommand{\xxspace}{\mathcal{X}}
\newcommand{\yyspace}{\mathcal{Y}}
\newcommand{\zzspace}{\mathcal{Z}}
\newcommand{\Ypred}{\widehat{Y}}
\newcommand{\errgap}{\Delta_{\err}}

\newcommand{\ber}[3]{\mathrm{BER}_{#1}(#2~\|~#3)}
\newcommand{\cegap}{\Delta_{\mathrm{CE}}}
\newcommand{\glsa}{GLS}
\newcommand{\dann}{\text{DANN}}
\newcommand{\cdan}{\text{CDAN}}
\newcommand{\jan}{\text{JAN}}
\newcommand{\iwdan}{\text{IWDAN}}
\newcommand{\iwcdan}{\text{IWCDAN}}
\newcommand{\iwjan}{\text{IWJAN}}
\newcommand{\iwdano}{\text{IWDAN-O}}
\newcommand{\iwcdano}{\text{IWCDAN-O}}
\newcommand{\iwjano}{\text{IWJAN-O}}
\newcommand{\dtv}{d_{\text{TV}}}

\newcommand*{\MinNumber}{0.0}%
\newcommand*{\MidNumber}{70} %
\newcommand*{\MaxNumber}{100}%

\newcommand{\ApplyGradient}[1]{%
        \ifdim #1 pt > \MidNumber pt
            \pgfmathsetmacro{\PercentColor}{max(min(100.0*(#1 - \MidNumber)/(\MaxNumber-\MidNumber),100.0),0.00)} %
            \hspace{-0.33em}\colorbox{green!\PercentColor!yellow}{#1}
        \else
            \pgfmathsetmacro{\PercentColor}{max(min(100.0*(\MidNumber - #1)/(\MidNumber-\MinNumber),100.0),0.00)} %
            \hspace{-0.33em}\colorbox{red!\PercentColor!yellow}{#1}
        \fi
}

\newcolumntype{R}{>{\collectcell\ApplyGradient}c<{\endcollectcell}}

\title{Domain Adaptation with Conditional Distribution Matching and Generalized Label Shift}

\author{%
  Remi Tachet des Combes\thanks{The first two authors contributed equally to this work. Work done while HZ was at Carnegie Mellon University.}\\
  Microsoft Research Montreal\\
  Montreal, QC, Canada \\
  \texttt{retachet@microsoft.com} \\
  \And
  Han Zhao\footnotemark[1]\\
  D. E. Shaw \& Co.\\
  New York, NY, USA \\
  \texttt{han.zhao@cs.cmu.edu} \\
  \AND
  Yu-Xiang Wang\\
  UC Santa Barbara\\
  Santa Barbara, CA, USA \\
  \texttt{yuxiangw@cs.ucsb.edu} \\
  \And
  Geoff Gordon\\
  Microsoft Research Montreal\\
  Montreal, QC, Canada \\
  \texttt{ggordon@microsoft.com} \\
}

\begin{document}
\maketitle
\begin{abstract}
    Adversarial learning has demonstrated good performance in the unsupervised domain adaptation setting, by learning domain-invariant representations. However, recent work has shown limitations of this approach when label distributions differ between the source and target domains. In this paper, we propose a new assumption, \textit{generalized label shift} ($\glsa$), to improve robustness against mismatched label distributions. $\glsa$ states that, conditioned on the label, there exists a representation of the input that is invariant between the source and target domains. Under $\glsa$, we provide theoretical guarantees on the transfer performance of any classifier. We also devise necessary and sufficient conditions for $\glsa$ to hold, by using an estimation of the relative class weights between domains and an appropriate reweighting of samples. Our weight estimation method could be straightforwardly and generically applied in existing domain adaptation (DA) algorithms that learn domain-invariant representations, with small computational overhead. In particular, we modify three DA algorithms, JAN, DANN and CDAN, and evaluate their performance on standard and artificial DA tasks. Our algorithms outperform the base versions, with vast improvements for large label distribution mismatches. Our code is available at \url{https://tinyurl.com/y585xt6j}.
\end{abstract}

\section{Introduction}
\label{sec:intro}
In spite of impressive successes, most deep learning models~\citep{goodfellow2017learning} rely on huge amounts of labelled data and their features have proven brittle to distribution shifts~\citep{yosinski2014transferable,linzen2019right}. Building more robust models, that learn from fewer samples and/or generalize better out-of-distribution is the focus of many recent works~\citep{journals/corr/abs-1906-00910,arjovsky2019invariant,forget}. The research direction of interest to this paper is that of domain adaptation, which aims at learning features that transfer well between domains. We focus in particular on unsupervised domain adaptation (UDA), where the algorithm has access to labelled samples from a source domain and unlabelled data from a target domain. Its objective is to train a model that generalizes well to the target domain. Building on advances in adversarial learning~\citep{goodfellow2014generative}, adversarial domain adaptation (ADA) leverages the use of a discriminator to learn an intermediate representation that is invariant between the source and target domains. Simultaneously, the representation is paired with a classifier, trained to perform well on the source domain~\citep{ganin2016domain,tzeng2017adversarial,zhao2018adversarial,LiuLWJ19}. ADA is rather successful on a variety of tasks, however, recent work has proven an upper bound on the performance of existing algorithms when source and target domains have mismatched label distributions~\citep{conf/icml/0002CZG19}. Label shift is a property of two domains for which the marginal label distributions differ, but the conditional distributions of input given label stay the same across domains~\citep{storkey,journals/corr/ZhangYCW15}. 

In this paper, we study domain adaptation under mismatched label distributions and design methods that are robust in that setting. Our contributions are the following. First, we extend the upper bound by~\citet{conf/icml/0002CZG19} to $k$-class classification and to conditional domain adversarial networks, a recently introduced domain adaptation algorithm~\citep{cdan}. Second, we introduce \textit{generalized label shift} ($\glsa$), a broader version of the standard label shift where conditional invariance between source and target domains is placed in representation rather than input space. Third, we derive performance guarantees for algorithms that seek to enforce $\glsa$ via learnt feature transformations, in the form of upper bounds on the error gap and the joint error of the classifier on the source and target domains. Those guarantees suggest principled modifications to ADA to improve its robustness to mismatched label distributions. The modifications rely on estimating the class ratios between source and target domains and use those as importance weights in the adversarial and classification objectives. The importance weights estimation is performed using a method of moment by solving a quadratic program, inspired from~\citet{lipton2018detecting}. Following these theoretical insights, we devise three new algorithms based on learning importance-weighted representations, $\dann$s~\citep{ganin2016domain}, $\jan$s~\citep{long2017deep} and $\cdan$s~\citep{cdan}. We apply our variants to artificial UDA tasks with large divergences between label distributions, and demonstrate significant performance gains compared to the algorithms' base versions. Finally, we evaluate them on standard domain adaptation tasks and also show improved performance.

\section{Preliminaries}
\label{sec:preliminary}
\textbf{Notation}~~We focus on the general $k$-class classification problem. $\data$ and $\ydata$ denote the input and output space, respectively. $\hdata$ stands for the representation space induced from $\data$ by a feature transformation $g:\data\mapsto\hdata$. Accordingly, we use $X, Y, Z$ to denote random variables which take values in $\data, \ydata,\hdata$. \emph{Domain} corresponds to a joint distribution on the input space $\data$ and output space $\ydata$, and we use $\dist_S$ (resp. $\dist_T$) to denote the source (resp. target) domain. Noticeably, this corresponds to a stochastic setting, which is stronger than the deterministic one previously studied~\citep{ben2007analysis,ben2010theory,conf/icml/0002CZG19}. A \emph{hypothesis} is a function $h:\data\to [k]$.  The \emph{error} of a hypothesis $h$ under distribution $\dist_S$ is defined as: $\eps_S(h)\defeq \Pr_{\dist_S}(h(X)\neq Y)$, i.e., the probability that $h$ disagrees with $Y$ under $\dist_S$.  

\textbf{Domain Adaptation via Invariant Representations}~~
For source ($\dist_S$) and target ($\dist_T$) domains, we use $\dist_S^X$, $\dist_T^X$, $\dist^Y_S$ and $\dist^Y_T$ to denote the marginal data and label distributions. In UDA, the algorithm has access to $n$ labeled points \mbox{$\{(\xx_i, y_i)\}_{i=1}^n\in(\data\times\ydata)^n$} and $m$ unlabeled points $\{\xx_j\}_{j=1}^m\in\data^m$ sampled i.i.d.\ from the source and target domains. Inspired by~\citet{ben2010theory}, a common approach is to learn representations invariant to the domain shift. With $g:\data\mapsto\hdata$ a feature transformation and $h:\hdata\mapsto\ydata$ a hypothesis on the feature space, a domain invariant representation~\citep{ganin2016domain,tzeng2017adversarial,zhao2018multiple} is a function $g$ that induces similar distributions on $\dist_S$ and $\dist_T$. $g$ is also required to preserve rich information about the target task so that $\err_S(h\circ g)$ is small. The above process results in the following Markov chain (assumed to hold throughout the paper): 
\begin{equation}\label{eq:markov_chain}
    X \overset{g}{\longrightarrow} Z \overset{h}{\longrightarrow} \Ypred,
\end{equation}
where $\Ypred = h(g(X))$. We let $\domain_S^Z$, $\domain_T^Z$, $\domain_S^{\Ypred}$ and $\domain_T^{\Ypred}$ denote the pushforwards (induced distributions) of $\domain_S^X$ and $\domain_T^X$ by $g$ and $h\circ g$. Invariance in feature space is defined as minimizing a distance or divergence between $\domain_S^Z$ and $\domain_T^Z$.

\textbf{Adversarial Domain Adaptation} \quad
Invariance is often attained by training a discriminator $d:\hdata\mapsto [0,1]$ to predict if $z$ is from the source or target. $g$ is trained both to maximize the discriminator loss and minimize the classification loss of $h\circ g$ on the source domain ($h$ is also trained with the latter objective). This leads to domain-adversarial neural networks~\citep[DANN]{ganin2016domain}, where $g$, $h$ and $d$ are parameterized with neural networks: $g_\theta$, $h_\phi$ and $d_\psi$ (see Algo.~\ref{alg:iwdan} and App.~\ref{sub:losses}). 
Building on $\dann$, conditional domain adversarial networks~\citep[CDAN]{cdan} use the same adversarial paradigm. However, the discriminator now takes as input the outer product, for a given $x$, between the predictions of the network $h(g(x))$ and its representation $g(x)$. In other words, $d$ acts on the outer product: $h \otimes g(x) \defeq (h_1(g(x))\cdot g(x),\dots, h_k(g(x))\cdot g(x))$
rather than on $g(x)$. $h_i$ denotes the $i$-th element of vector $h$. We now highlight a limitation of DANNs and CDANs.

\textbf{An Information-Theoretic Lower Bound}\quad
We let $\jsd$ denote the Jensen-Shanon divergence between two distributions (App.~\ref{sub:def}), and $\widetilde{Z}$ correspond to $Z$ (for $\dann$) or to $\widehat{Y} \otimes Z$ (for $\cdan$). The following theorem lower bounds the joint error of the classifier on the source and target domains:
\begin{restatable}{theorem}{ll}\label{thm:lower_bound}
\vspace*{-1em}
Assuming that $\jsd(\domain_S^Y~\|~\domain_T^Y)\geq \jsd(\domain_S^{\widetilde{Z}}~\|~ \domain_T^{\widetilde{Z}})$, then: 
\begin{align}\label{eq:lower_bound}
    \err_S(h\circ g) + \err_T(h\circ g) \geq
    \frac{1}{2}\left(\sqrt{\jsd(\domain_S^Y~\|~\domain_T^Y)} - \sqrt{\jsd(\domain_S^{\widetilde{Z}}~\|~\domain_T^{\widetilde{Z}})}\right)^2. \nonumber
\end{align}
\end{restatable}
\textbf{Remark}~~The lower bound is algorithm-independent. It is also a population-level result and holds asymptotically with increasing data. \citet{conf/icml/0002CZG19} prove the theorem for $k = 2$ and $\widetilde{Z} = Z$. We extend it to $\cdan$ and arbitrary $k$ (it actually holds for any $\widetilde{Z}$ s.t. $\Ypred = \widetilde{h}(\widetilde{Z})$ for some $\widetilde{h}$, see App.~\ref{sub:low_bound_generalization}). Assuming that label distributions differ between source and target domains, the lower bound shows that: \textit{the better the alignment of feature distributions, the worse the joint error.} For an invariant representation ($\jsd(\domain_S^{\tilde{Z}}, \domain_T^{\tilde{Z}}) = 0$) with no source error, the target error will be larger than $\jsd(\domain_S^Y, \domain_T^Y) / 2$. Hence algorithms learning invariant representations and minimizing the source error are fundamentally flawed when label distributions differ between source and target domains.

\begin{table*}[tb]
    \centering
    \caption{Common assumptions in the domain adaptation literature.}
    \label{tab:shift}
    \begin{tabular}{cc}\toprule
    \textbf{Covariate Shift} & \textbf{Label Shift} \\\midrule
    $\dist_S^X\neq \dist_T^X$  & $\dist_S^Y\neq \dist_T^Y$ \\
    $\forall\xx\in\xxspace, \dist_S(Y\mid X=\xx) = \dist_T(Y\mid X=\xx)$ & $\forall y\in\yyspace, \dist_S(X\mid Y = y) = \dist_T(X\mid Y = y)$\\\bottomrule
    \end{tabular}
\end{table*}
\textbf{Common Assumptions to Tackle Domain Adaptation}\quad Two common assumptions about the data made in DA are \emph{covariate shift} and \emph{label shift}. They correspond to different ways of decomposing the joint distribution over $X\times Y$, as detailed in Table~\ref{tab:shift}. From a representation learning perspective, covariate shift is not robust to feature transformation and can lead to an effect called negative transfer~\citep{conf/icml/0002CZG19}. At the same time, label shift clearly fails in most practical applications, e.g. transferring knowledge from synthetic to real images~\citep{visda}. In that case, the input distributions are actually disjoint.

\section{Main Results}
\label{sec:theory}
In light of the limitations of existing assumptions, (e.g. covariate shift and label shift), we propose \textit{generalized label shift} ($\glsa$), a relaxation of label shift that substantially improves its applicability. We first discuss some of its properties and explain why the assumption is favorable in domain adaptation based on representation learning. Motivated by $\glsa$, we then present a novel error decomposition theorem that directly suggests a bound minimization framework for domain adaptation. The framework is naturally compatible with $\FF$-integral probability metrics~\citep[$\FF$-IPM]{muller1997integral} and generates a family of domain adaptation algorithms by choosing various function classes $\FF$. In a nutshell, the proposed framework applies a method of moments~\citep{lipton2018detecting} to estimate the importance weight $\ww$ of the \emph{marginal label distributions} by solving a quadratic program (QP), and then uses $\ww$ to align the weighted source feature distribution with the target feature distribution.

\subsection{Generalized Label Shift}
\label{sec:label_shift}
\begin{definition}[Generalized Label Shift, $\glsa$]\label{def:glsa} A representation $Z = g(X)$ satisfies $\glsa$ if 
\begin{equation}\label{eq:glsa}
    \dist_S(Z\mid Y = y) = \dist_T(Z\mid Y=y),~\forall y\in\yyspace.
\end{equation}
\end{definition}
First, when $g$ is the identity map, the above definition of $\glsa$ reduces to the original label shift assumption. Next, $\glsa$ is always achievable for any distribution pair $(\dist_S, \dist_T)$: any constant function $g\equiv c\in\RR$ satisfies the above definition. The most important property is arguably that, unlike label shift, $\glsa$ is compatible with a perfect classifier (in the noiseless case). Precisely, if there exists a ground-truth labeling function $h^*$ such that $Y = h^*(X)$, then $h^*$ satisfies $\glsa$. As a comparison, without conditioning on $Y = y$, $h^*$ does not satisfy $\dist_S(h^*(X)) = \dist_T(h^*(X))$ if the marginal label distributions are different across domains. This observation is consistent with the lower bound in Theorem~\ref{thm:lower_bound}, which holds for arbitrary marginal label distributions.

$\glsa$ imposes label shift in the feature space $\zzspace$ instead of the original input space $\xxspace$. Conceptually, although samples from the same classes in the source and target domain can be dramatically different, the hope is to find an intermediate representation for both domains in which samples from a given class look similar to one another. Taking digit classification as an example and assuming the feature variable $Z$ corresponds to the contour of a digit, it is possible that by using different contour extractors for e.g. MNIST and USPS, those contours look roughly the same in both domains. Technically, $\glsa$ can be facilitated by having separate representation extractors $g_S$ and $g_T$ for source and target~\citep{bousmalis2016domain,tzeng2017adversarial}.

\subsection{An Error Decomposition Theorem based on \texorpdfstring{$\glsa$}{alpha}}
\label{sub:label_shift}
We now provide performance guarantees for models that satisfy $\glsa$, in the form of upper bounds on the error gap and on the joint error between source and target domains. It requires two concepts:
\begin{definition}[Balanced Error Rate] 
The \emph{balanced error rate} (BER) of predictor $\Ypred$ on $\dist_S$ is:
\begin{align}\label{eq:ber}
\ber{\dist_S}{\Ypred}{Y} \defeq \max_{j\in[k]} \dist_S(\Ypred \neq Y \mid Y=j).
\end{align}
\end{definition}

\begin{definition}[Conditional Error Gap]
Given a joint distribution $\dist$, the \emph{conditional error gap} of a classifier $\Ypred$ is $\cegap(\Ypred)\defeq \max_{y \neq y'\in\yyspace^2}~|\dist_S(\Ypred = y'\mid Y = y) - \dist_T(\Ypred = y'\mid Y = y)|$.
\end{definition}

When $\glsa$ holds, $\cegap(\Ypred)$ is equal to $0$. We now give an upper bound on the error gap between source and target, which can also be used to obtain a generalization upper bound on the target risk.
\begin{restatable}{theorem}{gap}(Error Decomposition Theorem)\label{thm:errorgap}
For any classifier $\Ypred = (h\circ g)(X)$, 
\begin{align*}
|\err_S(h\circ g) - \err_T(h\circ g)|\leq~ \|\dist_S^Y - \dist_T^Y\|_1\cdot\ber{\dist_S}{\Ypred}{Y} + 2(k-1)\cegap(\Ypred),
\end{align*}
where $\|\dist_S^Y - \dist_T^Y\|_1\defeq \sum_{i=1}^k |\dist_S(Y = i) - \dist_T(Y = i)|$ is the $L_1$ distance between $\dist_S^Y$ and $\dist_T^Y$.
\end{restatable}
\textbf{Remark}\quad The upper bound in Theorem~\ref{thm:errorgap} provides a way to decompose the error gap between source and target domains. It also immediately gives a generalization bound on the target risk $\err_T(h\circ g)$. The bound contains two terms. The first contains $\|\dist_S^Y - \dist_T^Y\|_1$, which measures the distance between the marginal label distributions across domains and is a constant that only depends on the adaptation problem itself, and $\mathrm{BER}$, a reweighted classification performance on the source domain. The second is $\cegap(\Ypred)$ measures the distance between the family of conditional distributions $\Ypred\mid Y$. In other words, the bound is \emph{oblivious} to the optimal labeling functions in feature space. This is in sharp contrast with upper bounds from previous work~\citep[Theorem 2]{ben2010theory}, \citep[Theorem 4.1]{conf/icml/0002CZG19}, which essentially decompose the error gap in terms of the distance between the marginal feature distributions ($\dist_S^Z$, $\dist_T^Z$) and the optimal labeling functions ($f_S^Z$, $f_T^Z$). Because the optimal labeling function in feature space depends on $Z$ and is unknown in practice, such decomposition is not very informative. As a comparison, Theorem~\ref{thm:errorgap} provides a decomposition orthogonal to previous results and does not require knowledge about unknown optimal labeling functions in feature space.

Notably, the balanced error rate, $\ber{\dist_S}{\Ypred}{Y}$, only depends on samples from the source domain and can be minimized. Furthermore, using a data-processing argument, the conditional error gap $\cegap(\Ypred)$ can be minimized by aligning the conditional feature distributions across domains. Putting everything together, the result suggests that, to minimize the error gap, it suffices to align the conditional distributions $Z\mid Y = y$ while simultaneously minimizing the balanced error rate. In fact, under the assumption that the conditional distributions are perfectly aligned (i.e., under $\glsa$), we can prove a stronger result, guaranteeing that the joint error is small:
\begin{restatable}{theorem}{joint}\label{thm:jointerror}
If $Z = g(X)$ satisfies $\glsa$, then for any $h:\zzspace\to\yyspace$ and letting $\Ypred = h(Z)$ be the predictor, we have $\err_{S}(\Ypred) + \err_{T}(\Ypred)\leq 2\ber{\dist_S}{\Ypred}{Y}$.
\end{restatable}

\subsection{Conditions for Generalized Label Shift}
\label{sec:conditions}
The main difficulty in applying a bound minimization algorithm inspired by Theorem~\ref{thm:errorgap} is that we do not have labels from the target domain in UDA, so we cannot directly align the conditional label distributions. By using relative class weights between domains, we can provide a necessary condition for $\glsa$ that bypasses an explicit alignment of the conditional feature distributions.
\begin{definition}\label{def:iw} 
Assuming $\dist_S(Y = y) > 0, \forall y\in \yyspace$, we let $\ww \in \mathbb{R}^{k}$ denote the importance weights of the target and source label distributions:
\vspace{-0.1cm}
\begin{align}
    &\ww_y\defeq \frac{\dist_T(Y = y)}{\dist_S(Y = y)}, \quad\forall y\in\yyspace. \label{eq:weights}
\end{align}
\end{definition}
\begin{restatable}{lemma}{reweight}(Necessary condition for $\glsa$) \label{lemma:necessary_condition} If $Z = g(X)$ satisfies $\glsa$, then $\dist_T(\widetilde{Z}) = \sum_{y\in\yyspace} \ww_y\cdot \dist_S(\widetilde{Z}, Y = y) \eqdef \dist_S^\ww(\widetilde{Z})$ where $\widetilde{Z}$ verifies either $\widetilde{Z} = Z$ or $\widetilde{Z} = \Ypred \otimes Z$. 
\end{restatable}
Compared to previous work that attempts to align $\dist_T(Z)$ with $\dist_S(Z)$ (using adversarial discriminators~\citep{ganin2016domain} or maximum mean discrepancy (MMD)~\citep{long2015learning}) or $\dist_T(\hat{Y} \otimes Z)$ with $\dist_S(\hat{Y} \otimes Z)$~\citep{cdan}, Lemma~\ref{lemma:necessary_condition} suggests to instead align $\dist_T(\widetilde{Z})$ with the \emph{reweighted} marginal distribution $\dist^{\ww}_S(\widetilde{Z})$. Reciprocally, the following two theorems give sufficient conditions to know when perfectly aligned target feature distribution and reweighted source feature distribution imply $\glsa$:
\begin{restatable}{theorem}{sufficient}(Clustering structure implies sufficiency)
\label{thm:sufficient}
Let $Z = g(X)$ such that $\dist_T(Z) = \dist_S^\ww(Z)$. Assume $\dist_T(Y = y) > 0,\forall y\in\yyspace$. If there exists a partition of $\zzspace = \cup_{y\in\yyspace}\zzspace_y$ such that $\forall y\in\yyspace$, $\dist_S(Z\in\zzspace_y\mid Y = y) = \dist_T(Z\in\zzspace_y\mid Y = y) = 1$, then $Z=g(X)$ satisfies $\glsa$.
\end{restatable}
\paragraph{Remark} Theorem~\ref{thm:sufficient} shows that if there exists a partition of the feature space such that instances with the same label are within the same component, then aligning the target feature distribution with the reweighted source feature distribution implies $\glsa$. While this clustering assumption may seem strong, it is consistent with the goal of reducing classification error: if such a clustering exists, then there also exists a perfect predictor based on the feature $Z = g(X)$, i.e., the cluster index.


\begin{restatable}{theorem}{sufficientcondition}\label{thm:sufficient_condition}
Let $\Ypred = h(Z)$, $\gamma\defeq \min_{y\in\yyspace} \dist_T(Y = y)$ and $\ww_M \defeq \max_{y\in\yyspace}\thinspace\ww_{y}$. For $\widetilde{Z} = Z$ or $\widetilde{Z} = \hat{Y} \otimes Z$, we have:
\vspace{-0.2cm}
\begin{align*}
    \max_{y\in\yyspace}~\dtv(\dist_S(Z\mid Y = y), \dist_T(Z\mid Y=y)) \leq \frac{\ww_M \err_{S}(\Ypred) + \err_{T}(\Ypred) + \sqrt{8\jsd(\dist_S^{\ww}(\widetilde{Z})\|\dist_T(\widetilde{Z}))} }{\gamma}.
\end{align*}
\end{restatable}
Theorem~\ref{thm:sufficient_condition} confirms that matching $\dist_T(\widetilde{Z})$ with $\dist^{\ww}_S(\widetilde{Z})$ is the proper objective in the context of mismatched label distributions. It shows that, for matched feature distributions and a source error equal to zero, successful domain adaptation (\textit{i.e.} a target error equal to zero) implies that $\glsa$ holds. Combined with Theorem~\ref{thm:jointerror}, we even get equivalence between the two.

\textbf{Remark}\quad Thm.~\ref{thm:sufficient_condition} extends Thm.~\ref{thm:sufficient} by incorporating the clustering assumption in the joint error achievable by a classifier $\Ypred$ based on a fixed $Z$. In particular, if the clustering structure holds, the joint error is 0 for an appropriate $h$, and aligning the reweighted feature distributions implies $\glsa$.

\subsection{Estimating the Importance Weights \texorpdfstring{$\ww$}{alpha}}
Inspired by the moment matching technique to estimate $\ww$ under label shift from~\citet{lipton2018detecting}, we propose a method to get $\ww$ under $\glsa$ by solving a quadratic program (QP).

\begin{definition}\label{def:confusion} We let $\textbf{C} \in \mathbb{R}^{|\yyspace| \times |\yyspace|}$ denote the confusion matrix of the classifier on the source domain and $\boldsymbol\mu \in \mathbb{R}^{|\yyspace|}$ the distribution of predictions on the target one, $\forall y,y'\in\yyspace$:
\begin{align}
    &\textbf{C}_{y,y'}\defeq \dist_S(\Ypred = y, Y = y'), \qquad
    \boldsymbol\mu_y \defeq \dist_T(\Ypred = y). \nonumber
\end{align}
\end{definition}
\begin{restatable}{lemma}{estimation}\label{lem:iw_equation}
If $\glsa$ is verified, and if the confusion matrix $\textbf{C}$ is invertible, then $\ww = \textbf{C}^{-1} \boldsymbol\mu$.
\end{restatable}

The key insight from Lemma~\ref{lem:iw_equation} is that, to estimate the importance vector $\ww$ under $\glsa$, we do not need access to labels from the target domain. However, matrix inversion is notoriously numerically unstable, especially with finite sample estimates $\hat{\textbf{C}}$ and $\hat{\boldsymbol\mu}$ of $\textbf{C}$ and $\boldsymbol\mu$. We propose to solve instead the following QP (written as $QP(\hat{\textbf{C}}, \hat{\boldsymbol\mu})$), whose solution will be consistent if $\hat{\textbf{C}} \to \textbf{C}$ and $\hat{\boldsymbol\mu}\to \boldsymbol\mu$:
\begin{equation}\label{qp}
    \underset{\ww}{\text{minimize}}\quad\frac{1}{2}~\|\hat{\boldsymbol\mu} - \hat{\textbf{C}} \ww\|_2^2,
    \quad\quad\quad\text{ subject to }\quad \ww \geq 0, ~\ww^T \dist_S(Y) = 1.
\end{equation}
The above program~\eqref{qp} can be efficiently solved in time $O(|\yyspace|^3)$, with $|\yyspace|$ small and constant; and by construction, its solution is element-wise non-negative, even with limited amounts of data to estimate $\textbf{C}$ and $\boldsymbol\mu$.

\begin{restatable}{lemma}{weightsconvergence}\label{lem:convergence}
If the source error $\err_S(h\circ g)$ is zero and the source and target marginals verify $\jsd(\dist_S^{\tilde{\ww}}(Z), \dist_T(Z)) = 0$, then the estimated weight vector $\ww$ is equal to $\tilde{\ww}$.
\end{restatable}
Lemma~\ref{lem:convergence} shows that the weight estimation is stable once the DA losses have converged, but it does not imply convergence to the true weights (see Sec.~\ref{sec:expe} and App.~\ref{sub:estimation} for more details).

\subsection{\texorpdfstring{$\FF$}{alpha}-IPM for Distributional Alignment}
\label{sec:ipm}

To align the target feature distribution and the reweighted source feature distribution as suggested by Lemma~\ref{lemma:necessary_condition}, we now provide a general framework using the integral probability metric~\citep[IPM]{muller1997integral}.
\begin{definition}
With $\FF$ a set of real-valued functions, the $\FF$-IPM between distributions $\dist$ and $\dist'$ is
\begin{equation}
    d_{\FF}(\dist,\dist')\defeq \sup_{f\in\FF}|\Exp_{X\sim\dist}[f(X)] - \Exp_{X\sim\dist'}[f(X)]|.
\label{equ:ipm}
\vspace*{-1em}
\end{equation}
\end{definition}
By approximating any function class $\FF$ using parametrized models, e.g., neural networks, we obtain a general framework for domain adaptation by aligning reweighted source feature distribution and target feature distribution, i.e. by minimizing $d_{\FF}(\dist_T(\widetilde{Z}), \dist_S^\ww(\widetilde{Z}))$. Below, by instantiating $\FF$ to be the set of bounded norm functions in a RKHS $\HH$~\citep{gretton2012kernel}, we obtain maximum mean discrepancy methods, leading to $\iwjan$ (cf. Section~\ref{sec:algo}), a variant of $\jan$~\citep{long2017deep} for UDA. 


\section{Practical Implementation}
\label{sec:imp}

\subsection{Algorithms}
\label{sec:algo}

The sections above suggest simple algorithms based on representation learning: (i) estimate $\ww$ on the fly during training, (ii) align the feature distributions $\widetilde{Z}$ of the target domain with the reweighted feature distribution of the source domain and, (iii) minimize the balanced error rate. Overall, we present the pseudocode of our algorithm in Alg.~\ref{alg:iwdan}.

To compute $\ww$, we build estimators $\hat{\textbf{C}}$ and $\hat{\boldsymbol\mu}$ of $\textbf{C}$ and $\boldsymbol\mu$ by averaging during each epoch the predictions of the classifier on the source (per true class) and target (overall). This corresponds to the inner-most loop of Algorithm \ref{alg:iwdan} (lines 8-9). At epoch end, $\ww$ is updated (line 10), and the estimators reset to 0. We have found empirically that using an exponential moving average of $\ww$ performs better. Our results all use a factor $\lambda = 0.5$. We also note that Alg.~\ref{alg:iwdan} implies a minimal computational overhead (see App.~\ref{sub:compute} for details): in practice our algorithms run as fast as their base versions.

\begin{algorithm}[tb]
  \caption{Importance-Weighted Domain Adaptation}
  \label{alg:iwdan}
    \begin{algorithmic}[1]
        \STATE {\bfseries Input:} source and target data $(x_S,y_S)$, $x_T$; $g_\theta$, $h_\phi$ and $d_\psi$; epochs $E$, batches per epoch $B$
        \STATE Initialize $\ww_1 = 1$
        \FOR{$t=1$ {\bfseries to} $E$}
        \STATE Initialize $\hat{\textbf{C}} = 0$, $\hat{\boldsymbol\mu} = 0$
        \FOR{$b=1$ {\bfseries to} $B$}
        \STATE Sample batches $(x^i_S,y^i_S)$ and $(x^i_T)$ of size s
        \STATE Maximize  $\mathcal{L}_{DA}^{\ww_t}$ w.r.t. $\theta$, minimize  $\mathcal{L}_{DA}^{\ww_t}$ w.r.t. $\psi$ and minimize  $\mathcal{L}_{C}^{\ww_t}$ w.r.t. $\theta$ and $\phi$
        \FOR{$i=1$ {\bfseries to} $s$}
        \STATE $\hat{\textbf{C}}_{\cdot y_S^i} \leftarrow \hat{\textbf{C}}_{\cdot y_S^i} + h_\phi(g_\theta(x_S^i))$ ($y_S^i$-th column) \quad and \quad $\hat{\boldsymbol\mu} \leftarrow \hat{\boldsymbol\mu} + h_\phi(g_\theta(x_T^i))$
        \ENDFOR
        \ENDFOR
        \STATE $\hat{\textbf{C}} \leftarrow \hat{\textbf{C}} / sB$ and $\hat{\boldsymbol\mu} \leftarrow \hat{\boldsymbol\mu} / sB$; \quad then \quad $\ww_{t+1} = \lambda\cdot QP(\hat{\textbf{C}}, \hat{\boldsymbol\mu}) + (1 - \lambda) \ww_t$
        \ENDFOR
    \end{algorithmic}
\end{algorithm}


Using $\ww$, we can define our first algorithm, \emph{Importance-Weighted Domain Adversarial Network} ($\iwdan$), that aligns $\dist^{\ww}_S(Z)$ and $\dist_T(Z)$) using a discriminator. To that end, we modify the $\dann$ losses $\mathcal{L}_{DA}$ and $\mathcal{L}_{C}$ as follows. For batches $(x^i_S,y^i_S)$ and $(x^i_T)$ of size $s$, the weighted DA loss is:\vspace{-0.2cm}
\begin{align}\label{eq:iwda_loss}
    \mathcal{L}_{DA}^{\ww}(x^i_S,y^i_S,x_T^i; \theta, \psi) = - \frac{1}{s} \displaystyle{\sum_{i = 1}^s} \ww_{y^i_S} \log(d_\psi(g_\theta(x_S^i)))
    + \log(1 - d_\psi(g_\theta(x_T^i))).
\end{align}
We verify in App.~\ref{lem:jsd}, that the standard ADA framework applied to $\mathcal{L}_{DA}^{\ww}$ indeed minimizes $\jsd(\dist^{\ww}_S(Z) \| \dist_T(Z))$. Our second algorithm, \emph{Importance-Weighted Joint Adaptation Networks} ($\iwjan$) is based on $\jan$~\citep{long2017deep} and follows the reweighting principle described in Section~\ref{sec:ipm} with $\FF$ a learnt RKHS (the exact $\jan$ and $\iwjan$ losses are specified in App. \ref{sub:losses}). Finally, our third algorithm is \emph{Importance-Weighted Conditional Domain Adversarial Network} ($\iwcdan$). It matches $\dist_S^{\ww}(\hat{Y} \otimes Z)$ with $\dist_T(\hat{Y} \otimes Z)$ by replacing the standard adversarial loss in $\cdan$ with Eq.~\ref{eq:iwda_loss}, where $d_\psi$ takes as input $(h_\phi \circ g_\theta) \otimes g_\theta$ instead of $g_\theta$.
The classifier loss for our three variants is:
\vspace{-0.1cm}
\begin{align}\label{eq:iwc_loss}
    \mathcal{L}_{C}^{\ww}(x^i_S,y^i_S; \theta, \phi) =
    -\frac{1}{s} \displaystyle{\sum_{i = 1}^s} \frac{1}{k\cdot\dist_S(Y = y)}\log(h_\phi(g_\theta(x_S^i))_{y^i_S}).
\end{align}
This reweighting is suggested by our theoretical analysis from Section \ref{sec:theory}, where we seek to minimize the balanced error rate $\ber{\dist_S}{\Ypred}{Y}$. We also define oracle versions, $\iwdano$, $\iwjano$ and $\iwcdano$, where the weights $\ww$ are the true weights. It gives an idealistic version of the reweighting method, and allows to assess the soundness of $\glsa$. $\iwdan$, $\iwjan$ and $\iwcdan$ are Alg.~\ref{alg:iwdan} with their respective loss functions, the oracle versions use the true weights instead of $\ww_t$.

\subsection{Experiments}
\label{sec:expe}
\begin{figure}[tb]
	\centering
	\includegraphics[width=0.49\linewidth]{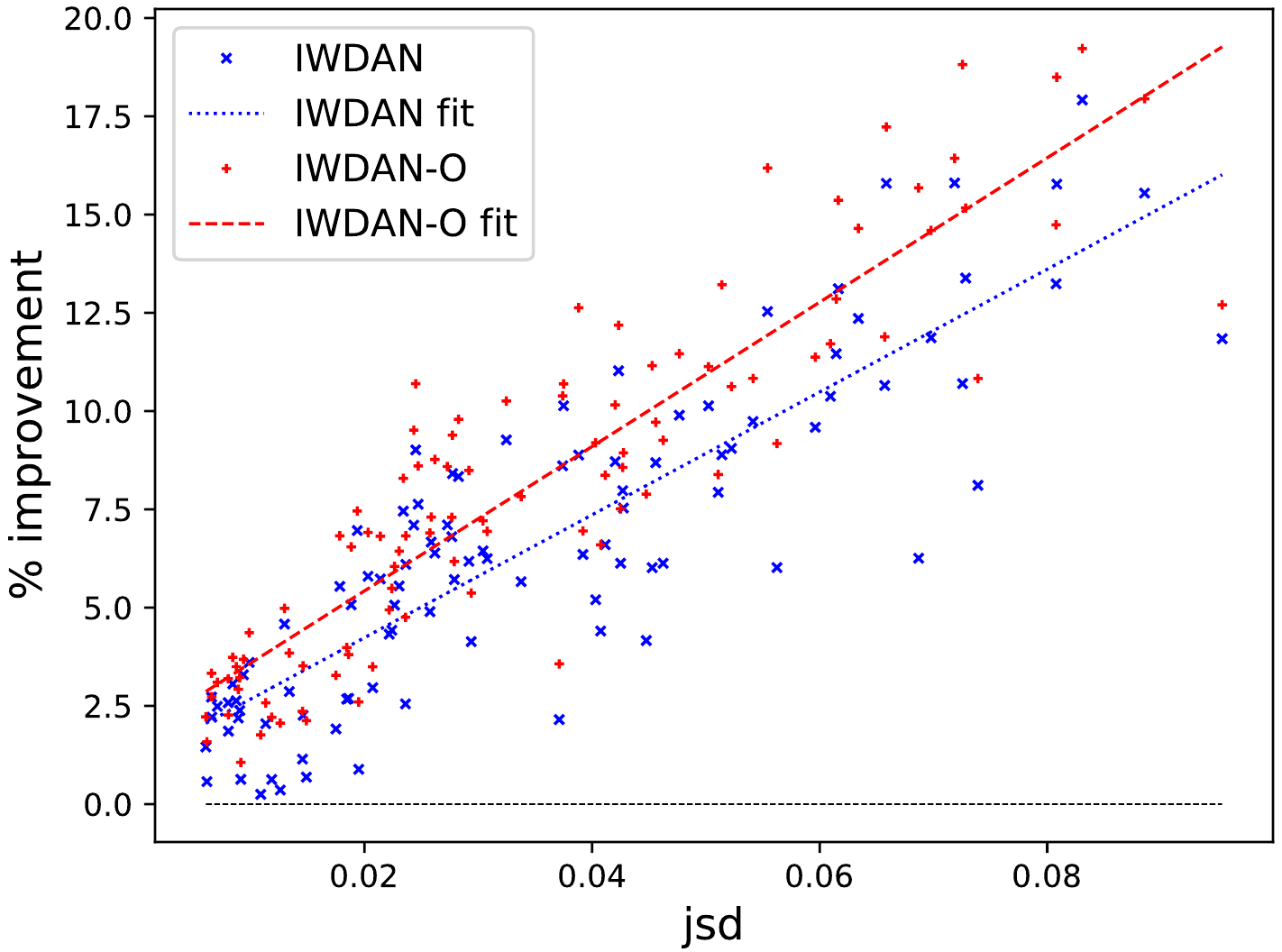}
	\includegraphics[width=0.48\linewidth]{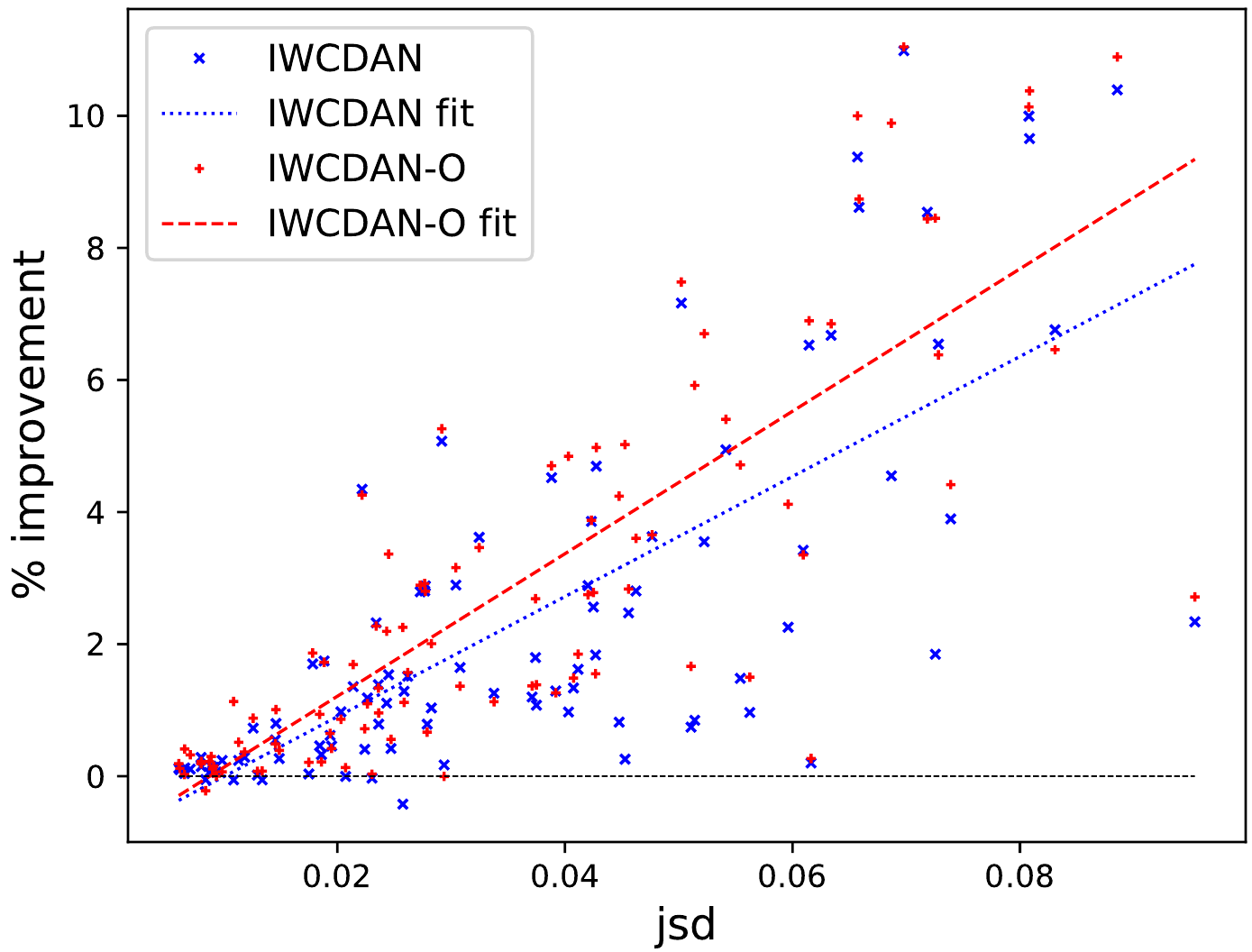}
	\caption{Gains of our algorithms vs their base versions (the horizontal grey line) for 100 tasks. The $x$-axis is \mbox{$\jsd(\domain_S^Y, \domain_T^Y)$}. The mean improvements for $\iwdan$ and $\iwdano$ (resp. $\iwcdan$ and $\iwcdano$) are $6.55\%$ and $8.14\%$ (resp. $2.25\%$ and $2.81\%$).}
	\label{fig:cloud}
\end{figure}

We apply our three base algorithms, their importance weighted versions, and the oracles to 4 standard DA datasets generating 21 tasks: Digits (MNIST $\leftrightarrow$ USPS~\citep{mnist,dua2017}), \citet{visda}, Office-31~\citep{conf/eccv/SaenkoKFD10} and Office-Home~\citep{home}. All values are averages over 5 runs of the best test accuracy throughout training (evaluated at the end of each epoch). We used that value for fairness with respect to the baselines (as shown in the left panel of Figure~\ref{fig:acc_distance}, the performance of $\dann$ decreases as training progresses, due to the inappropriate matching of representations showcased in Theorem~\ref{thm:lower_bound}). For full details, see App.~\ref{sub:datasets} and \ref{sub:imp}.

\textbf{Performance vs $\mathbf{\jsd}$}~~We artificially generate 100 tasks from MNIST and USPS by considering various random subsets of the classes in either the source or target domain (see Appendix~\ref{sub:jsd_generation} for details). These 100 DA tasks have a $\jsd(\domain_S^Y, \domain_T^Y)$ varying between $0$ and $0.1$. Applying $\iwdan$ and $\iwcdan$ results in Fig.~\ref{fig:cloud}. We see a clear correlation between the improvements provided by our algorithms and \mbox{$\jsd(\domain_S^Y, \domain_T^Y)$}, which is well aligned with Theorem \ref{thm:lower_bound}. Moreover, $\iwdan$ outperfoms $\dann$ on the $100$ tasks and $\iwcdan$ bests $\cdan$ on $94$. Even on small divergences, our algorithms do not suffer compared to their base versions.

\textbf{Original Datasets}~~The average results on each dataset are shown in Table~\ref{tab:full} (see App.\ref{sub:full} for the per-task breakdown). $\iwdan$ outperforms the basic algorithm $\dann$ by $1.75\%$, $1.64\%$, $1.16\%$ and $2.65\%$ on the Digits, Visda, Office-31 and Office-Home tasks respectively. Gains for $\iwcdan$ are more limited, but still present: $0.18\%$, $0.89\%$, $0.07\%$ and $1.07\%$ respectively. This might be explained by the fact that $\cdan$ enforces a weak form of $\glsa$ (App.~\ref{subsub:losses_cdan}). Gains for $\jan$ are $0.58\%$, $0.19\%$ and $0.19\%$. We also show the fraction of times (over all seeds and tasks) our variants outperform the original algorithms. Even for small gains, the variants provide consistent improvements. Additionally, the oracle versions show larger improvements, which strongly supports enforcing $\glsa$.
\begingroup
\setlength{\tabcolsep}{2pt} 
\begin{table*}[ht]
\caption{Average results on the various domains (Digits has 2 tasks, Visda 1, Office-31 6 and Office-Home 12). The prefix $s$ denotes the experiment where the source domain is subsampled to increase $\jsd(\domain_S^Y, \domain_T^Y)$. Each number is a mean over 5 seeds, the subscript denotes the fraction of times (out of $5~seeds \times \#tasks$) our algorithms outperform their base versions. $\jan$ is not available on Digits.}
\label{tab:full}
\begin{center}
\begin{footnotesize}
\begin{sc}
\begin{tabular}{lcc|cc|cc|cc}
\toprule
Method & Digits & $s$Digits & Visda & $s$Visda & O-31 & $s$O-31 & O-H & $s$O-H  \\
\midrule
No DA  & 77.17 & 75.67 & 48.39 & 49.02 & 77.81 & 75.72 & 56.39 & 51.34 \\
\midrule
DANN & 93.15 & 83.24 & 61.88 & 52.85 & 82.74 & 76.17 & 59.62 & 51.83 \\
IWDAN & \textbf{94.90}$_{100\%}$ & \textbf{92.54}$_{100\%}$ & \textbf{63.52}$_{100\%}$ & \textbf{60.18}$_{100\%}$ & \textbf{83.90}$_{87\%}$ & \textbf{82.60}$_{100\%}$ & \textbf{62.27}$_{97\%}$ & \textbf{57.61}$_{100\%}$ \\
IWDAN-O & 95.27$_{100\%}$ & 94.46$_{100\%}$ & 64.19$_{100\%}$ & 62.10$_{100\%}$ & 85.33$_{97\%}$ & 84.41$_{100\%}$ & 64.68$_{100\%}$ & 60.87$_{100\%}$ \\
\midrule
CDAN & 95.72 & 88.23 & 65.60 & 60.19 & 87.23 & 81.62 & 64.59 & 56.25 \\
IWCDAN & \textbf{95.90}$_{80\%}$ & \textbf{93.22}$_{100\%}$ & \textbf{66.49}$_{60\%}$ & \textbf{65.83}$_{100\%}$ & \textbf{87.30}$_{73\%}$ & \textbf{83.88}$_{100\%}$ & \textbf{65.66}$_{70\%}$ & \textbf{61.24}$_{100\%}$ \\
IWCDAN-O & 95.85$_{90\%}$ & 94.81$_{100\%}$ & 68.15$_{100\%}$ & 66.85$_{100\%}$ & 88.14$_{90\%}$ & 85.47$_{100\%}$ & 67.64$_{98\%}$ & 63.73$_{100\%}$ \\
\midrule
JAN & N/A & N/A & 56.98 & 50.64 & 85.13 & 78.21 & 59.59 & 53.94 \\
IWJAN & N/A & N/A & \textbf{57.56}$_{100\%}$ & \textbf{57.12}$_{100\%}$ & \textbf{85.32}$_{60\%}$ & \textbf{82.61}$_{97\%}$ & \textbf{59.78}$_{63\%}$ & \textbf{55.89}$_{100\%}$ \\
IWJAN-O & N/A & N/A & 61.48$_{100\%}$ & 61.30$_{100\%}$ & 87.14$_{100\%}$ & 86.24$_{100\%}$ & 60.73$_{92\%}$ & 57.36$_{100\%}$ \\
\bottomrule
\end{tabular}
\end{sc}
\end{footnotesize}
\end{center}
\vskip -0.1in
\vspace*{-0.5em}
\end{table*}
\endgroup
\begin{figure}[tb]
\centering
	\includegraphics[width=0.49\linewidth]{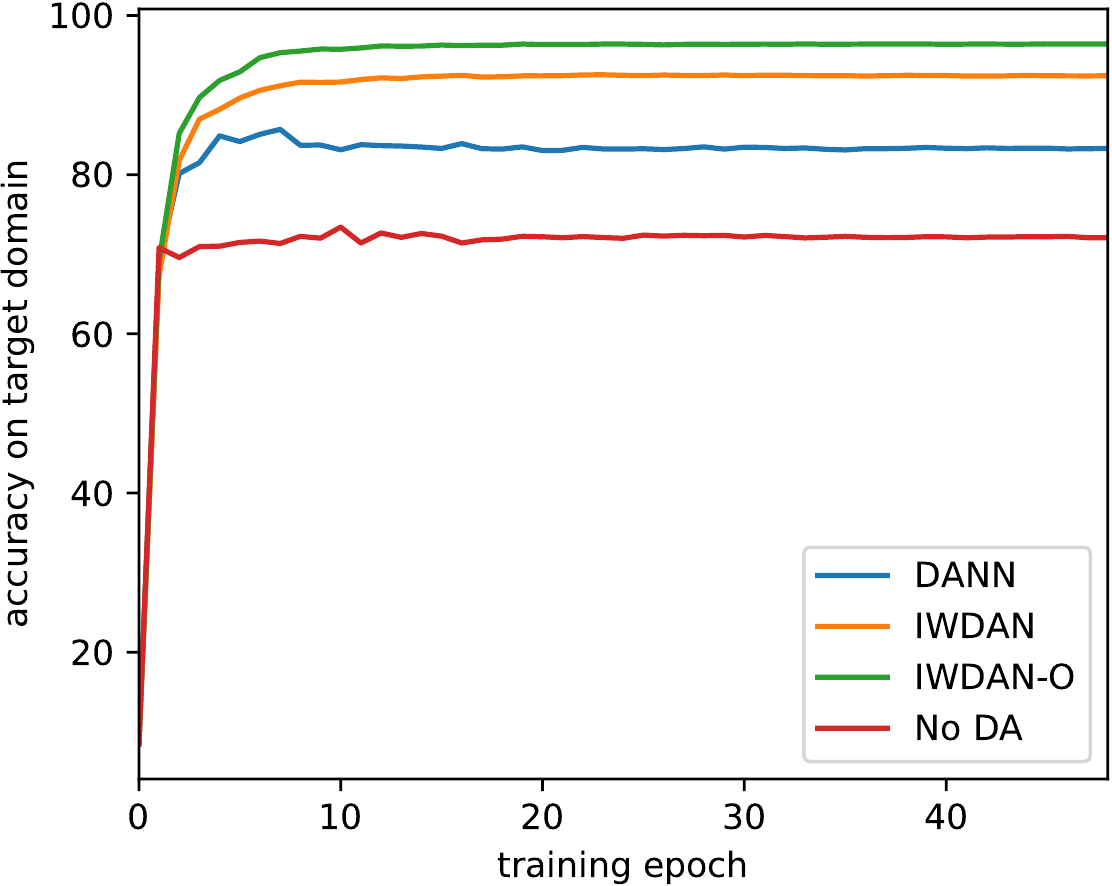}
	\includegraphics[width=0.49\linewidth]{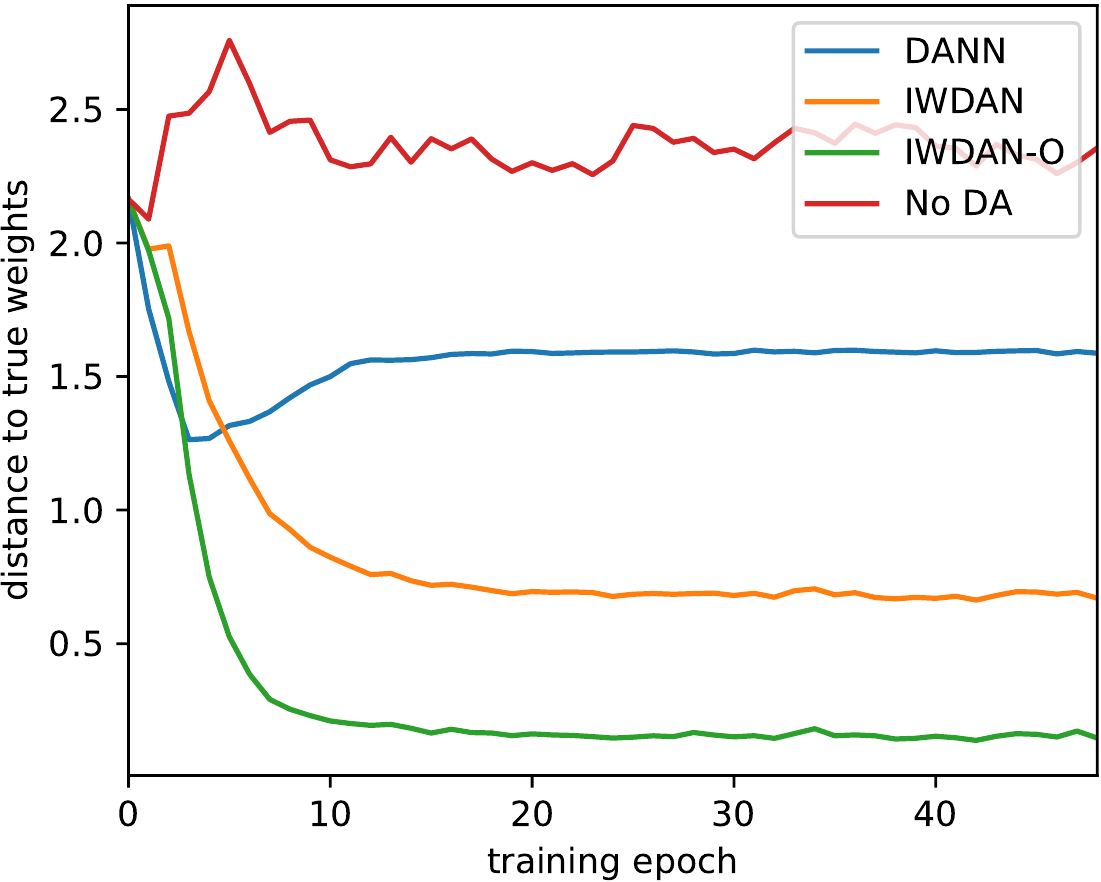}
	\caption{\textit{Left} Accuracy on sDigits. \textit{Right} Euclidian distance between estimated and true weights.}
	\label{fig:acc_distance}
\vspace*{-1em}
\end{figure}

\textbf{Subsampled datasets}~~The original datasets have fairly balanced classes, making the JSD between source and target label distributions $\jsd(\domain_S^Y~\|~\domain_T^Y)$ rather small (Tables \ref{tab:jsd_digits_full}, \ref{tab:jsd_office_full} and \ref{tab:jsd_home_full} in App.~\ref{sub:jsd}). To evaluate our algorithms on larger divergences, we arbitrarily modify the source domains above by considering only $30\%$ of the samples coming from the first half of their classes. This results in larger divergences (Tables \ref{tab:jsd_digits_sub}, \ref{tab:jsd_office_sub} and  \ref{tab:jsd_home_sub}). Performance is shown in Table~\ref{tab:full} (datasets prefixed by s). For $\iwdan$, we see gains of $9.3\%$, $7.33\%$, $6.43\%$ and $5.58\%$  on the digits, Visda, Office-31 and Office-Home datasets respectively. For $\iwcdan$, improvements are $4.99\%$, $5.64\%$, $2.26\%$ and $4.99\%$, and $\iwjan$ shows gains of $6.48\%$, $4.40\%$ and $1.95\%$. Moreover, on all seeds and tasks but one, our variants outperform their base versions.


\textbf{Importance weights} While our method demonstrates gains empirically, Lemma~\ref{lem:iw_equation} does not guarantee convergence of $\ww$ to the true weights. In Fig.~\ref{fig:acc_distance}, we show the test accuracy and distance between estimated and true weights during training on sDigits. We see that $\dann$'s performance gets worse after a few epoch, as predicted by Theorem~\ref{thm:lower_bound}. The representation matching objective collapses classes that are over-represented in the target domain on the under-represented ones (see App.~\ref{sub:confusion}). This phenomenon does not occur for $\iwdan$ and $\iwdano$. Both monotonously improve in accuracy and estimation (see Lemma~\ref{lem:convergence} and App.~\ref{sub:estimation} for more details). We also observe that $\iwdan$'s weights do not converge perfectly. This suggests that fine-tuning $\lambda$ (we used $\lambda = 0.5$ in all our experiments for simplicity) or updating $\ww$ more or less often could lead to better performance. 

\textbf{Ablation Study}~~Our algorithms have two components, a weighted adversarial loss $\mathcal{L}_{DA}^\ww$ and a weighted classification loss $\mathcal{L}_C^\ww$. In Table~\ref{tab:ablation}, we augment $\dann$ and $\cdan$ using those losses separately (with the true weights). We observe that $\dann$ benefits essentially from the reweighting of its adversarial loss $\mathcal{L}_{DA}^\ww$, the classification loss has little effect. For $\cdan$, gains are essentially seen on the subsampled datasets. Both losses help, with a $+2\%$ extra gain for $\mathcal{L}_{DA}^\ww$.
\begin{table}[tb]
    \centering
    \caption{Ablation study on the Digits tasks.}
    \label{tab:ablation}
    \begingroup
    \begin{tabular}{lcclcc}
        \toprule
        Method & Digits & sDigits & Method & Digits & sDigits \\
        \midrule
        DANN & 93.15 & 83.24 & CDAN & 95.72 & 88.23  \\
        DANN + $\mathcal{L}_{C}^{\ww}$  & 93.27 & 84.52 &  CDAN + $\mathcal{L}_{C}^{\ww}$  & 95.65 & 91.01 \\
        DANN + $\mathcal{L}_{DA}^{\ww}$ & \textbf{95.31} & \textbf{94.41} & CDAN + $\mathcal{L}_{DA}^{\ww}$ & 95.42 & 93.18  \\
        IWDAN-O  & \textbf{95.27} & \textbf{94.46} & IWCDAN-O & \textbf{95.85} & \textbf{94.81}  \\        \bottomrule
    \end{tabular}
    \endgroup
\vspace*{-0.5em}
\end{table}

\section{Related Work}
\label{sec:related}

Covariate shift has been studied and used in many adaptation algorithms~\citep{huang2006correcting,gretton2009covariate,ash2016unsupervised,adel2017unsupervised,tzeng2017adversarial,zhao2019deep,redko2019optimal}. While less known, label shift has also been tackled from various angles over the years: applying EM to learn $\dist^Y_T$~\citep{conf/ijcai/ChanN05}, placing a prior on the label distribution~\citep{storkey}, using kernel mean matching~\citep{,zhang2013domain,journals/nn/PlessisS14,conf/acml/NguyenPS15}, etc. \citet{conf/icml/ScholkopfJPSZM12} cast the problem in a causal/anti-causal perspective corresponding to covariate/label shift. That perspective was then further developed~\citep{zhang2013domain,gong2016domain,lipton2018detecting,conf/iclr/Azizzadenesheli19}. Numerous domain adaptation methods rely on learning invariant representations, and minimize various metrics on the marginal feature distributions: total variation or equivalently $\jsd$~\citep{ganin2016domain,tzeng2017adversarial,zhao2018adversarial,LiuLWJ19}, maximum mean discrepancy~\citep{gretton2012kernel,long2014transfer,long2015learning,long2016unsupervised,long2017deep}, Wasserstein distance~\citep{courty2017optimal,courty2017joint,shen2018wasserstein,lee2018minimax,chen2018re}, etc. Other noteworthy DA methods use reconstruction losses and cycle-consistency to learn transferable classifiers~\citep{conf/iccv/ZhuPIE17,hoffman2017cycada,conf/icml/XieZCC18}. Recently, \citet{LiuLWJ19} have introduced Transferable Adversarial Training (TAT), where transferable examples are generated to fill the gap in feature space between source and target domains, the datasets is then augmented with those samples. Applying our method to TAT is a future research direction.

Other relevant settings include partial ADA, i.e. UDA when target labels are a strict subset of the source labels / some components of $\ww$ are $0$~\citep{conf/cvpr/CaoL0J18,conf/eccv/CaoMLW18,conf/cvpr/CaoYLW019}. Multi-domain adaptation, where multiple source or target domains are given, is also very studied~\citep{mansour2009mixture,daume2009frustratingly,Nam_2016_CVPR,zhao2018multiple,guo2018multi,PengHSS19}. Recently, \citet{journals/corr/abs-1905-12760} study sample reweighting in the domain transfer to handle mass shifts between distributions.

Prior work on combining importance weight in domain-invariant representation learning also exists in the setting of partial DA~\citep{zhang2018importance}. However, the importance ratio in these works is defined over the features $Z$, rather than the class label $Y$. Compared to our method, this is both statistically inefficient and computationally expensive, since the feature space $\hdata$ is often a high-dimensional continuous space, whereas the label space $\ydata$ only contains a finite number ($k$) of distinct labels. In a separate work,~\citet{yan2017mind} proposed a weighted MMD distance to handle target shift in UDA. However, their weights are estimated based on pseudo-labels obtained from the learned classifier, hence it is not clear whether the pseudo-labels provide accurate estimation of the importance weights even in simple settings. As a comparison, under $\glsa$, we show that our weight estimation by solving a quadratic program converges asymptotically. 

\section{Conclusion and Future Work}
\label{sec:conclusion}
We have introduced the generalized label shift assumption, $\glsa$, and theoretically-grounded variations of existing algorithms to handle mismatched label distributions. On tasks from classic benchmarks as well as artificial ones, our algorithms consistently outperform their base versions. The gains, as expected theoretically, correlate well with the JSD between label distributions across domains. In real-world applications, the JSD is unknown, and might be larger than in ML datasets where classes are often purposely balanced. Being simple to implement and adding barely any computational cost, the robustness of our method to mismatched label distributions makes it very relevant to such applications.

\textbf{Extensions}~~The framework we define in this paper relies on appropriately reweighting the domain adversarial losses. It can be straightforwardly applied to settings where multiple source and/or target domains are used, by simply maintaining one importance weights vector $\ww$ for each source/target pair~\citep{zhao2018multiple,PengHSS19}. In particular, label shift could explain the observation from~\citet{zhao2018multiple} that too many source domains hurt performance, and our framework might alleviate the issue. One can also think of settings (e.g. semi-supervised domain adaptation) where estimations of $\dist_T^Y$ can be obtained via other means. A more challenging but also more interesting future direction is to extend our framework to \emph{domain generalization}, where the learner has access to multiple labeled source domains but no access to (even unlabelled) data from the target domain.

\section*{Acknowledgements}
The authors thank Romain Laroche and Alessandro Sordoni for useful feedback and helpful discussions. HZ and GG would like to acknowledge support from the DARPA XAI project, contract \#FA87501720152 and a Nvidia GPU grant.  YW would like acknowledge partial support from NSF Award \#2029626, a start-up grant from UCSB Department of Computer Science, as well as generous gifts from Amazon, Adobe, Google and NEC Labs.

\section*{Broader Impact}

Our work focuses on domain adaptation and attempts to properly handle mismatches in the label distributions between the source and target domains. Domain Adaptation as a whole aims at transferring knowledge gained from a certain domain (or data distribution) to another one. It can potentially be used in a variety of decision making systems, such as spam filters, machine translation, etc.. One can also potentially think of much more sensitive applications such as recidivism prediction, or loan approvals. 

While it is unclear to us to what extent DA is currently applied, or how it will be applied in the future, the bias formalized in Th.~\ref{thm:lower_bound} and verified in Table~\ref{tab:conf_dann} demonstrates that imbalances between classes will result in poor transfer performance of standard ADA methods on a subset of them, which is without a doubt a source of potential inequalities.
Our method is actually aimed at counter-balancing the effect of such imbalances. As shown in our empirical results (for instance Table~\ref{tab:conf_iwdan}) it is rather successful at it, especially on significant shifts. This makes us rather confident in the algorithm's ability to mitigate potential effects of biases in the datasets.
On the downside, failure in the weight estimation of some classes might result in poor performance on those. However, we have not observed, in any of our experiments, our method performing significantly worse than its base version. Finally, our method is a variation over existing deep learning algorithms. As such, it carries with it the uncertainties associated to deep learning models, in particular a lack of interpretability and of formal convergence guarantees.

\bibliographystyle{plainnat}
\bibliography{reference}

\clearpage
\appendix
\section{Omitted Proofs}
\label{sec:proofs}
In this section, we provide the theoretical material that completes the main text. 

\subsection{Definition}\label{sub:def}

\begin{definition}
Let us recall that for two distributions $\domain$ and $\domain'$, the Jensen-Shannon (JSD) divergence $\jsd(\domain~\|~\domain')$ is defined as:
\begin{equation*}
    \jsd(\domain~\|~\domain') \defeq \frac{1}{2}\kl(\domain~\|~\domain_M) + \frac{1}{2}\kl(\domain'~\|~\domain_M),
\end{equation*}
where $\kl(\cdot~\|~\cdot)$ is the Kullback–Leibler (KL) divergence and $\domain_M\defeq (\domain + \domain') / 2$. 
\end{definition}

\subsection{Consistency of the Weighted Domain Adaptation Loss \texorpdfstring{\eqref{eq:iwda_loss}}{alpha}}\label{sub:consistency}

For the sake of conciseness, we verify here that the domain adaptation training objective does lead to minimizing the Jensen-Shannon divergence between the weighted feature distribution of the source domain and the feature distribution of the target domain.
\begin{lemma}\label{lem:jsd} Let $p(x,y)$ and $q(x)$ be two density distributions, and $w(y)$ be a positive function such that $\int p(y) w(y) dy = 1$. Let $p^w(x) = \int p(x,y) w(y) dy$ denote the $w$-reweighted marginal distribution of $x$ under $p$. The minimum value of 
\begin{equation*}
    I(d) \defeq \mathbb{E}_{(x,y) \sim p, x'\sim q}[-w(y) \log(d(x)) - \log(1 - d(x'))]
\end{equation*} 
is $\log(4) - 2 \jsd(p^w(x)~\|~q(x))$, and is attained for $d^*(x) = \frac{p^w(x)}{p^w(x) + q(x)}$.
\begin{proof}
We see that:
\begin{align}
    I(d) &= - \iiint [w(y) \log(d(x)) + \log(1 - d(x'))] p(x,y) q(x') dx dx' dy \\
    &= - \int [\int w(y) p(x,y) dy] \log(d(x)) + q(x) \log(1 - d(x)) dx \\
    &= - \int p^w(x) \log(d(x)) + q(x) \log(1 - d(x)) dx.
\end{align}
From the last line, we follow the exact method from \citet{goodfellow2014generative} to see that point-wise in $x$ the minimum is attained for $d^*(x) = \frac{p^w(x)}{p^w(x) + q(x)}$ and that $I(d^*) = \log(4) - 2 \jsd(p^w(x)~\|~q(x))$.
\end{proof}
Applying Lemma \ref{lem:jsd} to $\dist_S(Z,Y)$ and $\dist_T(Z)$ proves that the domain adaptation objective leads to minimizing $\jsd(\dist^w_S(Z)~\|~\dist_T(Z))$.
\end{lemma}

\subsection{\texorpdfstring{$k$}{alpha}-class information-theoretic lower bound}\label{sub:low_bound_generalization}
In this section, we prove Theorem \ref{thm:lower_bound} that extends previous result to the general $k$-class classification problem.
\ll*
\begin{proof}
We essentially follow the proof from \citet{conf/icml/0002CZG19}, except for Lemmas 4.6 that needs to be adapted to the $\cdan$ framework and Lemma 4.7 to $k$-class classification. 

Lemma 4.6 from \citet{conf/icml/0002CZG19} states that $\jsd(\domain^{\widehat{Y}}_S,\domain^{\widehat{Y}}_T) \leq \jsd(\domain^Z_S, \domain^Z_T)$, which covers the case $\widetilde{Z} = Z$. 

When $\widetilde{Z} = \Ypred \otimes Z$, let us first recall that we assume $h$ or equivalently $\Ypred$ to be a one-hot prediction of the class. We have the following Markov chain:
\begin{equation*}
    X \overset{g}{\longrightarrow} Z \overset{\tilde{h}}{\longrightarrow} \widetilde{Z} \overset{l}{\longrightarrow} \Ypred, \label{eq:lower_bound_markov_chain}
\end{equation*}
where $\tilde{h}(z) = h(z) \otimes z$ and $l:\yyspace\otimes\hdata\to\yyspace$ returns the index of the non-zero block in $\tilde{h}(z)$. There is only one such block since $h$ is a one-hot, and its index corresponds to the class predicted by $h$. Given the definition of $l$, we clearly see that $\Ypred$ is independent of $X$ knowing $\widetilde{Z}$. We can now apply the same proof than in \citet{conf/icml/0002CZG19} to conclude that:
\begin{equation}\label{eq:generalized_DPI}
    \jsd(\domain^{\Ypred}_S,\domain^{\Ypred}_T) \leq \jsd(\domain^{\widetilde{Z}}_S, \domain^{\widetilde{Z}}_T).
\end{equation}
It essentially boils down to a data-processing argument: the discrimination distance between two distributions cannot increase after the same (possibly stochastic) channel (kernel) is applied to both. Here, the channel corresponds to the (potentially randomized) function $l$. 

\paragraph{Remark} Additionally, we note that the above inequality holds for \emph{any} $\tilde{Z}$ such that $\Ypred = l(\widetilde{Z})$ for a (potentially randomized) function l. This covers any and all potential combinations of representations at various layers of the deep net, including the last layer (which corresponds to its predictions $\Ypred$).

Let us move to the second part of the proof. We wish to show that $\jsd(\domain^Y, \domain^{\Ypred}) \leq \err(h\circ g)$, where $\domain$ can be either $\domain_S$ or $\domain_T$:
\begin{align}
    2\jsd(\domain^Y, \domain^{\widehat{Y}}) & \leq \|\domain^Y - \domain^{\widehat{Y}}\|_1 && \text{\citep{lin1991divergence}} \nonumber \\
    &= \displaystyle{\sum_{i=1}^{k}} |\domain(\widehat{Y}=i) - \domain(Y=i)| \nonumber \\
    &= \displaystyle{\sum_{i=1}^{k}} |\displaystyle{\sum_{j=1}^{k}} \domain(\widehat{Y}=i | Y=j) \domain(Y=j) - \domain(Y=i)| \nonumber \\
    &= \displaystyle{\sum_{i=1}^{k}} |\domain(\widehat{Y}=i | Y=i) \domain(Y=i) - \domain(Y=i) + \displaystyle{\sum_{j \neq i}} \domain(\widehat{Y}=i | Y=j) \domain(Y=j)| \nonumber \\ 
    &\leq \displaystyle{\sum_{i=1}^{k}} |\domain(\widehat{Y}=i | Y=i) - 1| \domain(Y=i) + \displaystyle{\sum_{i=1}^{k}} \displaystyle{\sum_{j \neq i}} \domain(\widehat{Y}=i | Y=j) \domain(Y=j) \nonumber \\
    &= \displaystyle{\sum_{i=1}^{k}} \domain(\widehat{Y} \neq Y | Y=i) \domain(Y=i) + \displaystyle{\sum_{j=1}^{k}} \displaystyle{\sum_{i \neq j}} \domain(\widehat{Y}=i | Y=j) \domain(Y=j) \nonumber \\
    &= 2 \displaystyle{\sum_{i=1}^{k}} \domain(\widehat{Y} \neq Y | Y=i) \domain(Y=i) = 2 \domain(\widehat{Y} \neq Y) = 2 \err(h\circ g). \label{eq:jsd_to_perf}
\end{align}
We can now apply the triangular inequality to $\sqrt{\jsd}$, which is a distance metric~\citep{endres2003new}, called the Jensen-Shannon distance. This gives us:
\begin{align*}
    \sqrt{\jsd(\domain_S^Y, \domain_T^Y)} & \leq \sqrt{\jsd(\domain_S^Y, \domain_S^{\widehat{Y}})} + \sqrt{\jsd(\domain_S^{\widehat{Y}}, \domain_T^{\widehat{Y}})} + \sqrt{\jsd(\domain_T^{\widehat{Y}}, \domain_T^Y)} \\
    & \leq \sqrt{\jsd(\domain_S^Y, \domain_S^{\widehat{Y}})} + \sqrt{\jsd(\domain^{\widetilde{Z}}_S, \domain^{\widetilde{Z}}_T)} + \sqrt{\jsd(\domain_T^{\widehat{Y}}, \domain_T^Y)} \\
    & \leq \sqrt{\err_S(h\circ g)} + \sqrt{\jsd(\domain^{\widetilde{Z}}_S, \domain^{\widetilde{Z}}_T)} + \sqrt{\err_T(h\circ g)}.
\end{align*}
where we used Equation \eqref{eq:generalized_DPI} for the second inequality and \eqref{eq:jsd_to_perf} for the third.

Finally, assuming that $\jsd(\domain_S^Y, \domain_T^Y)\geq \jsd(\domain_S^{\widetilde{Z}}, \domain_T^{\widetilde{Z}})$, we get:
\begin{align*}
    \left(\sqrt{\jsd(\domain_S^Y, \domain_T^Y)} - \sqrt{\jsd(\domain^{\widetilde{Z}}_S, \domain^{\widetilde{Z}}_T)} \right)^2
    \leq \left(\sqrt{\err_S(h\circ g)} + \sqrt{\err_T(h\circ g)} \right)^2 \leq 2 \left(\err_S(h\circ g) + \err_T(h\circ g)\right).
\end{align*}
which concludes the proof.
\end{proof}

\subsection{Proof of Theorem~\ref{thm:errorgap}}
To simplify the notation, we define the error gap $\errgap(\Ypred)$ as follows:
\begin{equation*}
    \errgap(\Ypred)\defeq |\err_S(\Ypred) - \err_T(\Ypred)|.
\end{equation*}
Also, in this case we use $\dist_a,~ a\in\{S, T\}$ to mean the source and target distributions respectively. Before we give the proof of Theorem~\ref{thm:errorgap}, we first prove the following two lemmas that will be used in the proof.
\begin{lemma}
\label{lemma:first}
Define $\gamma_{a,j}\defeq \dist_a(Y = j), \forall a\in\{S, T\}, \forall j\in[k]$, then $\forall \alpha_j, \beta_j \geq 0$ such that $\alpha_j + \beta_j = 1$, and $\forall i\neq j$, the following upper bound holds:
\begin{align*}
&|\gamma_{S,j} \dist_S(\Ypred = i\mid Y = j) - \gamma_{T,j}\dist_T(\Ypred = i\mid Y = j)| \leq \\ 
&\hspace{0.8cm} |\gamma_{S,j} - \gamma_{T,j}|\cdot\left(\alpha_j\dist_S(\Ypred = i\mid Y = j) + \beta_j\dist_T(\Ypred = i\mid Y = j)\right)
+ \gamma_{S,j}\beta_j \cegap(\Ypred) + \gamma_{T,j}\alpha_j \cegap(\Ypred).
\end{align*}
\end{lemma}
\begin{proof}
To make the derivation uncluttered, define $\dist_j(\Ypred = i)\defeq \alpha_j\dist_S(\Ypred = i\mid Y = j) + \beta_j\dist_T(\Ypred = i\mid Y = j)$ to be the mixture conditional probability of $\Ypred = i$ given $Y = j$, where the mixture weight is given by $\alpha_j$ and $\beta_j$. Then in order to prove the upper bound in the lemma, it suffices if we give the desired upper bound for the following term
\begin{align*}
    & \left||\gamma_{S,j} \dist_S(\Ypred = i\mid Y = j) - \gamma_{T,j}\dist_T(\Ypred = i\mid Y = j)| - |(\gamma_{S,j} - \gamma_{T,j}) \dist_j(\Ypred = i)| \right|    \\
    &\leq  \left|\left(\gamma_{S,j} \dist_S(\Ypred = i\mid Y = j) - \gamma_{T,j}\dist_T(\Ypred = i\mid Y = j)\right) - (\gamma_{S,j} - \gamma_{T,j}) \dist_j(\Ypred = i) \right|    \\
    &= \left|\gamma_{S,j} (\dist_S(\Ypred = i\mid Y = j) - \dist_j(\Ypred = i)) - \gamma_{T,j} (\dist_T(\Ypred = i\mid Y = j) - \dist_j(\Ypred = i))\right|,
\end{align*}
following which we will have:
\begin{align*}
    &~ |\gamma_{S,j} \dist_S(\Ypred = i\mid Y = j) - \gamma_{T,j}\dist_T(\Ypred = i\mid Y = j)| \leq |(\gamma_{S,j} - \gamma_{T,j}) \dist_j(\Ypred = i)| \\
    & + \left|\gamma_{S,j} (\dist_S(\Ypred = i\mid Y = j) - \dist_j(\Ypred = i)) - \gamma_{T,j} (\dist_T(\Ypred = i\mid Y = j) - \dist_j(\Ypred = i))\right| \\
    &\leq |\gamma_{S,j} - \gamma_{T,j}|\left(\alpha_j\dist_S(\Ypred = i\mid Y = j) + \beta_j\dist_T(\Ypred = i\mid Y = j)\right) \\
    &+ \gamma_{S,j}\left|\dist_S(\Ypred = i\mid Y = j) - \dist_j(\Ypred = i)\right| + \gamma_{T,j}\left|\dist_T(\Ypred = i\mid Y = j) - \dist_j(\Ypred = i)\right|.
\end{align*}
To proceed, let us first simplify $\dist_S(\Ypred = i\mid Y = j) - \dist_j(\Ypred = i)$. By definition of $\dist_j(\Ypred = i) = \alpha_j\dist_S(\Ypred = i\mid Y = j) + \beta_j \dist_T(\Ypred = i\mid Y = j)$, we know that:
\begin{align*}
    &~\dist_S(\Ypred = i\mid Y = j) - \dist_j(\Ypred = i) \\
    =&~ \dist_S(\Ypred = i\mid Y = j) - \big(\alpha_j\dist_S(\Ypred = i\mid Y = j) + \beta_j \dist_T(\Ypred = i\mid Y = j)\big) \\
    =&~ \big(\dist_S(\Ypred = i\mid Y = j) - \alpha_j\dist_S(\Ypred = i\mid Y = j)\big) - \beta_j\dist_T(\Ypred = i\mid Y = j) \\
    =&~ \beta_j\big(\dist_S(\Ypred = i\mid Y = j) - \dist_T(\Ypred = i\mid Y = j)\big).
\end{align*}
Similarly, for the second term $\dist_T(\Ypred = i\mid Y = j) - \dist_j(\Ypred = i)$, we can show that:
\begin{equation*}
    \dist_T(\Ypred = i\mid Y = j) - \dist_j(\Ypred = i) = \alpha_j\big(\dist_T(\Ypred = i\mid Y = j) - \dist_S(\Ypred = i\mid Y = j)\big).
\end{equation*}
Plugging these two identities into the above, we can continue the analysis with
\begin{align*}
     & \left|\gamma_{S,j} (\dist_S(\Ypred = i\mid Y = j) - \dist_j(\Ypred = i)) - \gamma_{T,j} (\dist_T(\Ypred = i\mid Y =j) - \dist_j(\Ypred = i))\right| \\
    &= \left|\gamma_{S,j} \beta(\dist_S(\Ypred = i\mid Y = j) - \dist_T(\Ypred = i\mid Y = j)) - \gamma_{T,j} \alpha_j(\dist_T(\Ypred = i\mid Y = j) - \dist_S(\Ypred = i\mid Y = j))\right| \\
    &\leq \left|\gamma_{S,j} \beta_j(\dist_S(\Ypred = i\mid Y = j) - \dist_T(\Ypred = i\mid Y = j))\right| + \left|\gamma_{T,j} \alpha_j(\dist_T(\Ypred = i\mid Y = j) - \dist_S(\Ypred = i\mid Y = j))\right| \\
    & \leq \gamma_{S,j} \beta_j\cegap(\Ypred) + \gamma_{T,j} \alpha_j\cegap(\Ypred).
\end{align*}
The first inequality holds by the triangle inequality and the second by the definition of the conditional error gap. Combining all the inequalities above completes the proof.
\end{proof}
We are now ready to prove the theorem:
\gap*
\begin{proof}[Proof of Theorem~\ref{thm:errorgap}]
First, by the law of total probability, it is easy to verify that following identity holds for $a\in\{S, T\}$:
\begin{align*}
    \dist_a(\Ypred\neq Y) &= \sum_{i\neq j}\dist_a(\Ypred = i, Y =j) = \sum_{i\neq j}\gamma_{a,j}\dist_a(\Ypred = i\mid Y = j).
\end{align*}
Using this identity, to bound the error gap, we have:
\begin{align*}
    &~|\dist_S(Y\neq \Ypred) - \dist_T(Y\neq \Ypred)| \\
    =&~ \big|\sum_{i\neq j}\gamma_{S,j}\dist_S(\Ypred = i\mid Y = j) - \sum_{i\neq j}\gamma_{T,j}\dist_T(\Ypred = i\mid Y = j)\big| \\
    \leq&~ \sum_{i\neq j}\big|\gamma_{S,j}\dist_S(\Ypred = i\mid Y = j) - \gamma_{T,j}\dist_T(\Ypred = i\mid Y = j)\big|.
\end{align*}
Invoking Lemma~\ref{lemma:first} to bound the above terms, and since $\forall j\in[k], \gamma_{S,j}, \gamma_{T,j}\in[0, 1]$, $\alpha_j + \beta_j = 1$, we get:
\begin{align*}
    &~ |\dist_S(Y\neq \Ypred) - \dist_T(Y\neq \Ypred)|  \\
    \leq&~ \sum_{i\neq j}\big|\gamma_{S,j}\dist_S(\Ypred = i\mid Y = j) - \gamma_{T,j}\dist_T(\Ypred = i\mid Y = j)\big| \\
    \leq&~ \sum_{i\neq j} |\gamma_{S,j} - \gamma_{T,j}|\cdot\left(\alpha_j\dist_S(\Ypred = i\mid Y = j) + \beta_j\dist_T(\Ypred = i\mid Y = j)\right) + \gamma_{S,j}\beta_j \cegap(\Ypred) + \gamma_{T,j}\alpha_j \cegap(\Ypred) \\
    \leq&~ \sum_{i\neq j} |\gamma_{S,j} - \gamma_{T,j}|\cdot\left(\alpha_j\dist_S(\Ypred = i\mid Y = j) + \beta_j\dist_T(\Ypred = i\mid Y = j)\right) + \gamma_{S,j} \cegap(\Ypred) + \gamma_{T,j} \cegap(\Ypred) \\
    =&~ \sum_{i\neq j} |\gamma_{S,j} - \gamma_{T,j}|\cdot\left(\alpha_j\dist_S(\Ypred = i\mid Y = j) + \beta_j\dist_T(\Ypred = i\mid Y = j)\right) + \sum_{i=1}^k\sum_{j\neq i}\gamma_{S,j} \cegap(\Ypred) + \gamma_{T,j} \cegap(\Ypred) \\
    =&~ \sum_{i\neq j} |\gamma_{S,j} - \gamma_{T,j}|\cdot\left(\alpha_j\dist_S(\Ypred = i\mid Y = j) + \beta_j\dist_T(\Ypred = i\mid Y = j)\right) + 2(k-1)\cegap(\Ypred). \\
    \intertext{Note that the above holds $\forall \alpha_j, \beta_j \geq 0$ such that $\alpha_j + \beta_j = 1$. By choosing $\alpha_j = 1,\forall j\in[k]$ and $\beta_j = 0, \forall j\in[k]$, we have:}
    =&~ \sum_{i\neq j} |\gamma_{S,j} - \gamma_{T,j}|\cdot\dist_S(\Ypred = i\mid Y = j) + 2(k-1)\cegap(\Ypred) \\
    =&~ \sum_{j = 1}^k |\gamma_{S,j} - \gamma_{T,j}|\cdot\left(\sum_{i=1, i\neq j}^k\dist_S(\Ypred = i\mid Y = j)\right) + 2(k-1)\cegap(\Ypred) \\
    =&~ \sum_{j = 1}^k |\gamma_{S,j} - \gamma_{T,j}|\cdot\dist_S(\Ypred \neq Y\mid Y = j) + 2(k-1)\cegap(\Ypred) \\
    \leq&~\|\dist_S^Y - \dist_T^Y\|_1\cdot\ber{\dist_S}{\Ypred}{Y} + 2(k-1)\cegap(\Ypred),
\end{align*}
where the last line is due to Holder's inequality, completing the proof.
\end{proof}

\subsection{Proof of Theorem~\ref{thm:jointerror}}

\joint*
\begin{proof}
First, by the law of total probability, we have:
\begin{align*}
    \err_{S}(\Ypred) + \err_{T}(\Ypred) &=  \dist_S(Y\neq \Ypred) + \dist_T(Y\neq \Ypred) \\
    &= \displaystyle{\sum_{j=1}^{k}} \displaystyle{\sum_{i \neq j}} \dist_S(\widehat{Y}=i | Y=j) \dist_S(Y=j) + \dist_T(\widehat{Y}=i | Y=j) \dist_T(Y=j).
    \intertext{Now, since $\Ypred = (h \circ g)(X) = h(Z)$, $\Ypred$ is a function of $Z$. Given the generalized label shift assumption, this guarantees that:
    \begin{align*}
        \forall y, y' \in \yyspace,\quad\dist_S(\Ypred = y'\mid Y = y) = \dist_T(\Ypred = y'\mid Y = y).
    \end{align*}
    Thus:
    }
    \err_{S}(\Ypred) + \err_{T}(\Ypred) &=  \displaystyle{\sum_{j=1}^{k}} \displaystyle{\sum_{i \neq j}} \dist_S(\widehat{Y}=i | Y=j) (\dist_S(Y=j) + \dist_T(Y=j)) \\
    &=\sum_{j\in[k]}\dist_S(\Ypred\neq Y\mid Y = j)\cdot (\dist_S(Y=j) + \dist_T(Y=j))\\
    &\leq \max_{j\in[k]}\dist_S(\Ypred\neq Y\mid Y = j)\cdot \sum_{j\in[k]}\dist_S(Y=j) + \dist_T(Y=j)\\
    &= 2\ber{\dist_S}{\Ypred}{Y}.\qedhere
\end{align*}
\end{proof}

\subsection{Proof of Lemma~\ref{lemma:necessary_condition}}
\reweight*
\begin{proof} From $\glsa$, we know that $\dist_S(Z\mid Y = y) = \dist_T(Z\mid Y = y)$. Applying any function $\tilde{h}$ to $Z$ will maintain that equality (in particular $\tilde{h}(Z) = \tilde{Y} \otimes Z$). Using that fact and Eq.~\eqref{eq:weights} on the second line gives:
\begin{align}
    \dist_T(\tilde{Z}) =&~ \sum_{y\in\yyspace} \dist_T(Y = y)\cdot \dist_T(\tilde{Z}\mid Y = y) \nonumber \\
    =&~ \sum_{y\in\yyspace} \ww_y\cdot \dist_S(Y = y)\cdot \dist_S(\tilde{Z}\mid Y = y) \nonumber \\
    =&~ \sum_{y\in\yyspace} \ww_y\cdot \dist_S(\tilde{Z}, Y = y). \qedhere
\end{align}
\end{proof}

\subsection{Proof of Theorem~\ref{thm:sufficient}}
\sufficient*
\begin{proof}
Follow the condition that $\dist_T(Z) = \dist_S^{\ww}(Z)$, by definition of $\dist_S^\ww(Z)$, we have:
\begin{align*}
    &\dist_T(Z) = \sum_{y\in\yyspace}\frac{\dist_T(Y = y)}{\dist_S(Y = y)}\dist_S(Z, Y = y) \\ \iff & \dist_T(Z) = \sum_{y\in\yyspace}\dist_T(Y =  y)\dist_S(Z\mid Y = y) \\
    \iff & \sum_{y\in\yyspace}\dist_T(Y =  y)\dist_T(Z\mid Y = y) = \sum_{y\in\yyspace}\dist_T(Y =  y)\dist_S(Z\mid Y = y).
\end{align*}
Note that the above equation holds for all measurable subsets of $\zzspace$. Now by the assumption that $\zzspace = \cup_{y\in\yyspace}\zzspace_y$ is a partition of $\zzspace$, consider $\zzspace_{y'}$:
\begin{equation*}
    \sum_{y\in\yyspace}\dist_T(Y =  y)\dist_T(Z\in\zzspace_{y'}\mid Y = y) = \sum_{y\in\yyspace}\dist_T(Y =  y)\dist_S(Z\in\zzspace_{y'}\mid Y = y).
\end{equation*}
Due to the assumption $\dist_S(Z\in\zzspace_y\mid Y = y) = \dist_T(Z\in\zzspace_y\mid Y = y) = 1$, we know that $\forall y'\neq y$, $\dist_T(Z\in\zzspace_{y'}\mid Y = y) = \dist_S(Z\in\zzspace_{y'}\mid Y = y) = 0$. This shows that both the supports of $\dist_S(Z\mid Y = y)$ and $\dist_T(Z\mid Y = y)$ are contained in $\zzspace_y$. Now consider an arbitrary measurable set $E\subseteq\zzspace_y$, since $\cup_{y\in\yyspace}\zzspace_y$ is a partition of $\zzspace$, we know that
\begin{equation*}
    \dist_S(Z\in E\mid Y = y') = \dist_T(Z\in E\mid Y = y') = 0, \quad\forall y'\neq y.
\end{equation*}
Plug $Z\in E$ into the following identity:
\begin{align*}
    &\sum_{y\in\yyspace}\dist_T(Y =  y)\dist_T(Z\in E\mid Y = y) = \sum_{y\in\yyspace}\dist_T(Y =  y)\dist_S(Z\in E\mid Y = y) \\
    \implies&~ \dist_T(Y =  y)\dist_T(Z\in E\mid Y = y) = \dist_T(Y =  y)\dist_S(Z\in E\mid Y = y) \\
    \implies&~ \dist_T(Z\in E\mid Y = y) = \dist_S(Z\in E\mid Y = y),
\end{align*}
where the last line holds because $\dist_T(Y = y) \neq 0$. Realize that the choice of $E$ is arbitrary, this shows that $\dist_S(Z\mid Y = y) = \dist_T(Z\mid Y = y)$, which completes the proof.
\end{proof}

\subsection{Sufficient Conditions for \texorpdfstring{$\glsa$}{alpha}}\label{sub:sufficient_condition}

\sufficientcondition*

\begin{proof}
To prove the above upper bound, let us first fix a $y\in\yyspace$ and fix a classifier $\Ypred = h(Z)$ for some $h:\zzspace\to\yyspace$. Now consider any measurable subset $E\subseteq\zzspace$, we would like to upper bound the following quantity:
\begin{align*}
    |\dist_S(Z \in E \mid Y = y) &- \dist_T(Z \in E \mid Y=y)| \\
    &= \frac{1}{\dist_T(Y=y)}\cdot |\dist_S(Z \in E, Y = y) \ww_y - \dist_T(Z \in E, Y = y)| \\
    &\leq \frac{1}{\gamma}\cdot |\dist_S(Z \in E, Y = y) \ww_y - \dist_T(Z \in E, Y = y)|.
\end{align*}
Hence it suffices if we can upper bound $|\dist_S(Z \in E, Y = y) \ww_y - \dist_T(Z \in E, Y = y)|$. To do so, consider the following decomposition:
\begin{align*}
     |\dist_T(Z \in E, Y = y) - \dist_S(Z \in E, Y = y) \ww_y| =&~ |\dist_T(Z \in E, Y = y) - \dist_T(Z\in E, \Ypred = y)  \\
    &+ \dist_T(Z \in E, \Ypred = y) - \dist_S^{\ww}(Z\in E, \Ypred = y) \\
    &+ \dist_S^{\ww}(Z\in E, \Ypred = y) - \dist_S(Z \in E, Y = y) \ww_y| \\
    \leq&~|\dist_T(Z \in E, Y = y) - \dist_T(Z\in E, \Ypred = y)|  \\
    &+ |\dist_T(Z \in E, \Ypred = y) - \dist_S^{\ww}(Z\in E, \Ypred = y)| \\
    &+ |\dist_S^{\ww}(Z\in E, \Ypred = y) - \dist_S(Z \in E, Y = y) \ww_y|. \\
\end{align*}
We bound the above three terms in turn. First, consider $|\dist_T(Z \in E, Y = y) - \dist_T(Z\in E, \Ypred = y)|$:
\begin{align*}
    |\dist_T(Z \in E, Y = y) &- \dist_T(Z\in E, \Ypred = y)| \\
    =&~ |\sum_{y'}\dist_T(Z \in E, Y = y, \Ypred = y') - \sum_{y'}\dist_T(Z\in E, \Ypred = y, Y = y')| \\
    \leq&~ \sum_{y'\neq y} |\dist_T(Z \in E, Y = y, \Ypred = y') - \dist_T(Z\in E, \Ypred = y, Y = y')| \\
    \leq&~\sum_{y'\neq y}\dist_T(Z \in E, Y = y, \Ypred = y') + \dist_T(Z\in E, \Ypred = y, Y = y') \\
    \leq&~\sum_{y'\neq y}\dist_T(Y = y, \Ypred = y') + \dist_T(\Ypred = y, Y = y') \\
    \leq&~\dist_T(Y \neq \Ypred) \\
    =&~ \err_{T}(\Ypred),
\end{align*}
where the last inequality is due to the fact that the definition of error rate corresponds to the sum of all the off-diagonal elements in the confusion matrix while the sum here only corresponds to the sum of all the elements in two slices. Similarly, we can bound the third term as follows:
\begin{align*}
    &|\dist_S^{\ww}(Z\in E, \Ypred = y) - \dist_S(Z \in E, Y = y) \ww_y| \\ 
    &\hspace{3cm}= |\displaystyle{\sum_{y'}} \dist_S(Z \in E, \Ypred = y, Y = y') \ww_{y'} - \displaystyle{\sum_{y'}} \dist_S(Z \in E, \Ypred = y', Y = y) \ww_y| \\
    &\hspace{3cm}\leq |\displaystyle{\sum_{y' \neq y}} \dist_S(Z \in E, \Ypred = y, Y = y') \ww_{y'} - \dist_S(Z \in E, \Ypred = y', Y = y) \ww_y|  \\ 
    &\hspace{3cm}\leq \ww_M \displaystyle{\sum_{y' \neq y}} \dist_S(Z \in E, \Ypred = y, Y = y') + \dist_S(Z \in E, \Ypred = y', Y = y)  \\
    &\hspace{3cm}\leq \ww_M \dist_S(Z \in E, \Ypred \neq Y)  \\
    &\hspace{3cm}\leq \ww_M \err_{S}(\Ypred). 
\end{align*}
Now we bound the last term. Recall the definition of total variation, we have:
\begin{align*}
    |\dist_T(Z \in E, \Ypred = y) &- \dist_S^{\ww}(Z\in E, \Ypred = y)| \\
    &= |\dist_T(Z \in E\land Z\in \Ypred^{-1}(y)) - \dist_S^{\ww}(Z\in E \land Z\in\Ypred^{-1}(y))| \\
    &\leq \sup_{E'\text{ is measurable}}|\dist_T(Z \in E') - \dist_S^{\ww}(Z\in E')| \\
    &= \dtv(\dist_T(Z), \dist_S^\ww(Z)).
\end{align*}
Combining the above three parts yields
\begin{align*}
    |\dist_S(Z \in E \mid Y = y) - \dist_T(Z \in E \mid Y=y)| \leq \frac{1}{\gamma}\cdot \left(\ww_M \err_{S}(\Ypred)
    + \err_{T}(\Ypred) + \dtv(\dist_S^{\ww}(Z), \dist_T(Z)) \right).
\end{align*}
Now realizing that the choice of $y\in\yyspace$ and the measurable subset $E$ on the LHS is arbitrary, this leads to
\begin{align*}
    \max_{y\in\yyspace}\sup_{E}|\dist_S(Z \in E \mid Y = y) &- \dist_T(Z \in E \mid Y=y)| \\
    &\leq \frac{1}{\gamma}\cdot \left(\ww_M \err_{S}(\Ypred)
    + \err_{T}(\Ypred) + \dtv(\dist_S^{\ww}(Z), \dist_T(Z)) \right).
\end{align*}
From~\citet{PhysRevA.79.052311}, we have: 
\begin{equation*}
\dtv(\dist_S^{\ww}(Z), \dist_T(Z)) \leq \sqrt{8\jsd(\dist_S^{\ww}(Z)~||~\dist_T(Z))}
\end{equation*}
(the total variation and Jensen-Shannon distance are equivalent), which gives the results for $\tilde{Z} = Z$.
Finally, noticing that $z \to h(z) \otimes z$ is a bijection ($h(z)$ sums to 1), we have:
\begin{equation*}
    \jsd(\dist_S^{\ww}(Z)~||~\dist_T(Z)) = \jsd(\dist_S^{\ww}(\Ypred \otimes Z)~||~\dist_T(\Ypred \otimes Z)),
\end{equation*}
which completes the proof.
\end{proof}

Furthermore, since the above upper bound holds for any classifier $\Ypred = h(Z)$, we even have:
\begin{align*}
    \max_{y\in\yyspace}\dtv(\dist_S(Z \in E \mid Y = y)&, \dist_T(Z \in E \mid Y=y)) \\ 
    &\leq \frac{1}{\gamma}\cdot \inf_{\Ypred}\left(\ww_M \err_{S}(\Ypred)
    + \err_{T}(\Ypred) + \dtv(\dist_S^{\ww}(Z), \dist_T(Z)) \right).
\end{align*}

\subsection{Proof of Lemma \ref{lem:iw_equation}}
\label{sec:estimation}

\estimation*
\begin{proof}
Given~\eqref{eq:glsa}, and with the joint hypothesis $\Ypred = h(Z)$ over both source and target domains, it is straightforward to see that the induced conditional distributions over predicted labels match between the source and target domains, i.e.:
\begin{align}\label{eq:predlsa}
    \dist_S(\Ypred = h(Z)\mid Y = y) = \dist_T(\Ypred = h(Z)\mid Y = y),~\forall y\in\yyspace.
\end{align}
This allows us to compute $\boldsymbol\mu_y,~\forall y \in\yyspace$ as 
\begin{align*}
\dist_T(\Ypred = y) =&~ \sum_{y'\in\yyspace}\dist_T(\Ypred = y\mid Y = y')\cdot\dist_T(Y = y') \\
=&~ \sum_{y'\in\yyspace}\dist_S(\Ypred = y\mid Y = y')\cdot\dist_T(Y = y') \\
=&~ \sum_{y'\in\yyspace}\dist_S(\Ypred = y, Y = y')\cdot\frac{\dist_T(Y = y')}{\dist_S(Y=y')} \\
=&~ \sum_{y'\in\yyspace}\textbf{C}_{y,y'}\cdot \ww_{y'}.
\end{align*}
where we used \eqref{eq:predlsa} for the second line. We thus have \mbox{$\boldsymbol\mu = \textbf{C} \ww$} which concludes the proof.
\end{proof}

\subsection{\texorpdfstring{$\FF$}{alpha}-IPM for Distributional Alignment}
\label{sec:app_ipm}

In Table~\ref{tab:ipm}, we list different instances of IPM with different choices of the function class $\FF$ in the above definition, including the total variation distance, Wasserstein-1 distance and the Maximum mean discrepancy~\citep{gretton2012kernel}.
\begin{table}[htb]
    \centering
    \caption{List of IPMs with different $\FF$. $\|\cdot\|_{\text{Lip}}$ denotes the Lipschitz seminorm and $\HH$ is a reproducing kernel Hilbert space (RKHS).\vspace{0.3cm}}
    \label{tab:ipm}
    \begin{tabular}{cc}\toprule
    $\FF$     &  $d_\FF$ \\\midrule
    $\{f: \|f\|_\infty\leq 1\}$  & Total Variation  \\
    $\{f: \|f\|_{\text{Lip}}\leq 1\}$ & Wasserstein-1 distance\\
    $\{f: \|f\|_{\HH}\leq 1\}$ & Maximum mean discrepancy \\
    \bottomrule
    \end{tabular}
\end{table}

\section{Experimentation Details}
\label{sec:additional}

\subsection{Computational Complexity}\label{sub:compute}

Our algorithms imply negligible time and memory overhead compared to their base versions. They are, in practice, indistinguishable from the underlying baseline:
\begin{itemize}
    \item Weight estimation requires storing the confusion matrix $C$ and the predictions $\mu$. This has a memory cost of $O(k^2)$, small compared to the size of a neural network that performs well on k classes.
    \item The extra computational cost comes from solving the quadratic program~\ref{qp}, which only depends on the number of classes $k$ and is solved once per epoch (not per gradient step). For Office-Home, it is a $65 \times 65$ QP, solved $\approx 100$ times. Its runtime is negligible compared to tens of thousands of gradient steps.
\end{itemize}

\subsection{Description of the domain adaptation tasks}\label{sub:datasets}

\textbf{Digits} We follow a widely used evaluation protocol \citep{hoffman2017cycada,cdan}. For the digits datasets MNIST~(M, \citet{mnist}) and USPS~(U, \citet{dua2017}), we consider the DA tasks: \mbox{M $\rightarrow$ U} and \mbox{U $\rightarrow$ M}. Performance is evaluated on the 10,000/2,007 examples of the MNIST/USPS test sets. 

\textbf{\citet{visda}} is a sim-to-real domain adaptation task. The synthetic domain contains 2D rendering of 3D models captured at different angles and lighting conditions. The real domain is made of natural images. Overall, the training, validation and test domains contain 152,397, 55,388 and 5,534 images, from 12 different classes. 

\textbf{Office-31} \citep{conf/eccv/SaenkoKFD10} is one of the most popular dataset for domain adaptation . It contains 4,652 images from 31 classes. The samples come from three domains: Amazon (A), DSLR (D) and Webcam (W), which generate six possible transfer tasks, A $\rightarrow$ D, A $\rightarrow$ W, D $\rightarrow$ A, D $\rightarrow$ W, W $\rightarrow$ A and W $\rightarrow$ D, which we all evaluate.

\textbf{Office-Home} \citep{home} is a more complex dataset than Office-31. It consists of 15,500 images from 65 classes depicting objects in office and home environments. The images form four different domains: Artistic (A), Clipart (C), Product (P), and Real-World images (R). We evaluate the 12 possible domain adaptation tasks.

\subsection{Full results on the domain adaptation tasks}\label{sub:full}

Tables~\ref{tab:digits}, \ref{tab:visda}, \ref{tab:office_full}, \ref{tab:office_sub}, \ref{tab:home_full} and \ref{tab:home_sub} show the detailed results of all the algorithms on each task of the domains described above. The performance we report is the best test accuracy obtained during training over a fixed number of epochs. We used that value for fairness with respect to the baselines (as shown in Figure~\ref{fig:acc_distance} Left, the performance of DANN decreases as training progresses, due to the inappropriate matching of representations showcased in Theorem~\ref{thm:lower_bound}).

The subscript denotes the fraction of seeds for which our variant outperforms the base algorithm. More precisely, by outperform, we mean that for a given seed (which fixes the network initialization as well as the data being fed to the model) the variant has a larger accuracy on the test set than its base version. Doing so allows to assess specifically the effect of the algorithm, all else kept constant.

\begin{table*}[ht]
\caption{Results on the Digits tasks. M and U stand for MNIST and USPS, the prefix $s$ denotes the experiment where the source domain is subsampled to increase $\jsd(\domain_S^Y, \domain_T^Y)$.}
\label{tab:digits}
\vskip 0.15in
\begin{center}
\begin{small}
\begin{sc}
\begin{tabular}{lcc|c@{\hskip 0.5cm}||@{\hskip 0.5cm}cc|c@{\hskip 0.5cm}}
\toprule
Method & M $\rightarrow$ U & U $\rightarrow$ M & Avg.& sM $\rightarrow$ U & sU $\rightarrow$ M & Avg. \\
\midrule
No Ad.    & 79.04 & 75.30 & 77.17 & 76.02 & 75.32 & 75.67 \\
\midrule
DANN    & 90.65 & 95.66 & 93.15 & 79.03 & 87.46 & 83.24 \\
IWDAN    & \textbf{93.28}$_{100\%}$ & \textbf{96.52}$_{100\%}$ & \textbf{94.90}$_{100\%}$ & \textbf{91.77}$_{100\%}$ & \textbf{93.32}$_{100\%}$ & \textbf{92.54}$_{100\%}$ \\
IWDAN-O  & 93.73$_{100\%}$ & 96.81$_{100\%}$ & 95.27$_{100\%}$ & 92.50$_{100\%}$ & 96.42$_{100\%}$ & 94.46$_{100\%}$ \\
\midrule
CDAN    & 94.16 & 97.29 & 95.72 & 84.91 & 91.55 & 88.23 \\
IWCDAN  & \textbf{94.36}$_{60\%}$ & \textbf{97.45}$_{100\%}$ & \textbf{95.90}$_{80\%}$ & \textbf{93.42}$_{100\%}$ & \textbf{93.03}$_{100\%}$ & \textbf{93.22}$_{100\%}$ \\
IWCDAN-O & 94.34$_{80\%}$ & 97.35$_{100\%}$ & 95.85$_{90\%}$ & 93.37$_{100\%}$ & 96.26$_{100\%}$ & 94.81$_{100\%}$ \\
\bottomrule
\end{tabular}
\end{sc}
\end{small}
\end{center}
\vskip -0.1in
\end{table*}

\begin{table*}[ht]
\caption{Results on the Visda domain. The prefix $s$ denotes the experiment where the source domain is subsampled to increase $\jsd(\domain_S^Y, \domain_T^Y)$.}
\label{tab:visda}
\vskip 0.15in
\begin{center}
\begin{small}
\begin{sc}
\begin{tabular}{lc@{\hskip 0.5cm}||@{\hskip 0.5cm}c}
\toprule
Method &  Visda & sVisda \\
\midrule
No Ad.  & 48.39 & 49.02 \\
\midrule
DANN & 61.88 & 52.85 \\
IWDAN  & \textbf{63.52}$_{100\%}$ & \textbf{60.18}$_{100\%}$ \\
IWDAN-O  &  64.19$_{100\%}$ & 62.10$_{100\%}$ \\
\midrule
CDAN    & 65.60 & 60.19 \\
IWCDAN  & \textbf{66.49}$_{60\%}$ & \textbf{65.83}$_{100\%}$ \\
IWCDAN-O & 68.15$_{100\%}$ & 66.85$_{100\%}$ \\
\midrule
JAN    & 56.98$_{100\%}$ & 50.64$_{100\%}$ \\
IWJAN  & \textbf{57.56}$_{100\%}$ & \textbf{57.12}$_{100\%}$ \\
IWJAN-O & 61.48$_{100\%}$ & 61.30$_{100\%}$ \\
\bottomrule
\end{tabular}
\end{sc}
\end{small}
\end{center}
\vskip -0.1in
\end{table*}

\begin{table*}[ht]
\caption{Results on the Office dataset.}
\label{tab:office_full}
\vskip 0.15in
\begin{center}
\begin{small}
\begin{sc}
\begin{tabular}{lcccccc|c}
\toprule
Method & A $\rightarrow$ D & A $\rightarrow$ W & D $\rightarrow$ A & D $\rightarrow$ W & W $\rightarrow$ A & W $\rightarrow$ D & Avg.  \\
\midrule
No DA & 79.60 & 73.18 & 59.33 & 96.30 & 58.75 & 99.68 & 77.81 \\
\midrule
DANN & 84.06 & 85.41 & 64.67 & 96.08 & 66.77 & 99.44 & 82.74 \\
IWDAN & \textbf{84.30}$_{60\%}$ & \textbf{86.42}$_{100\%}$ & \textbf{68.38}$_{100\%}$ & \textbf{97.13}$_{100\%}$ & \textbf{67.16}$_{60\%}$ & \textbf{100.0}$_{100\%}$ & \textbf{83.90}$_{87\%}$ \\
IWDAN-O & 87.23$_{100\%}$ & 88.88$_{100\%}$ & 69.92$_{100\%}$ & 98.09$_{100\%}$ & 67.96$_{80\%}$ & 99.92$_{100\%}$ & 85.33$_{97\%}$ \\
\midrule
CDAN & \textbf{89.56} & 93.01 & 71.25 & 99.24 & 70.32 & 100.0 & 87.23 \\
IWCDAN & 88.91$_{60\%}$ & \textbf{93.23}$_{60\%}$ & \textbf{71.90}$_{80\%}$ & \textbf{99.30}$_{80\%}$ & \textbf{70.43}$_{60\%}$ & \textbf{100.0}$_{100\%}$ & \textbf{87.30}$_{73\%}$ \\
IWCDAN-O & 90.08$_{60\%}$ & 94.52$_{100\%}$ & 73.11$_{100\%}$ & 99.30$_{80\%}$ & 71.83$_{100\%}$ & 100.0$_{100\%}$ & 88.14$_{90\%}$ \\
\midrule
JAN & 85.94 & \textbf{85.66} & \textbf{70.50} & 97.48 & \textbf{71.5}  & 99.72& 85.13 \\
IWJAN & \textbf{87.68}$_{100\%}$ & 84.86$_{0\%}$ & 70.36$_{60\%}$ & \textbf{98.98}$_{100\%}$ & 70.06$_{0\%}$ & \textbf{100.0}$_{100\%}$ & \textbf{85.32}$_{60\%}$ \\
IWJAN-O & 89.68$_{100\%}$ & 89.18$_{100\%}$ & 71.96$_{100\%}$ & 99.02$_{100\%}$ & 73.0$_{100\%}$ & 100.0$_{100\%}$ & 87.14$_{100\%}$ \\
\bottomrule
\end{tabular}
\end{sc}
\end{small}
\end{center}
\vskip -0.1in
\end{table*}

\begin{table*}[ht]
\caption{Results on the Subsampled Office dataset.}
\label{tab:office_sub}
\vskip 0.15in
\begin{center}
\begin{small}
\begin{sc}
\begin{tabular}{lcccccc|c}
\toprule
Method & sA $\rightarrow$ D & sA $\rightarrow$ W & sD $\rightarrow$ A & sD $\rightarrow$ W & sW $\rightarrow$ A & sW $\rightarrow$ D & Avg.  \\
\midrule
No DA & 75.82 & 70.69 & 56.82 & 95.32 & 58.35 & 97.31 & 75.72 \\
\midrule
DANN & 75.46 & 77.66 & 56.58 & 93.76 & 57.51 & 96.02 & 76.17 \\
IWDAN & \textbf{81.61}$_{100\%}$ & \textbf{88.43}$_{100\%}$ & \textbf{65.00}$_{100\%}$ & \textbf{96.98}$_{100\%}$ & \textbf{64.86}$_{100\%}$ & \textbf{98.72}$_{100\%}$ & \textbf{82.60}$_{100\%}$ \\
IWDAN-O & 84.94$_{100\%}$ & 91.17$_{100\%}$ & 68.44$_{100\%}$ & 97.74$_{100\%}$ & 64.57$_{100\%}$ & 99.60$_{100\%}$ & 84.41$_{100\%}$ \\
\midrule
CDAN & 82.45 & 84.60 & 62.54 & 96.83 & 65.01 & 98.31 & 81.62 \\
IWCDAN & \textbf{86.59}$_{100\%}$ & \textbf{87.30}$_{100\%}$ & \textbf{66.45}$_{100\%}$ & \textbf{97.69}$_{100\%}$ & \textbf{66.34}$_{100\%}$ & \textbf{98.92}$_{100\%}$ & \textbf{83.88}$_{100\%}$ \\
IWCDAN-O & 87.39$_{100\%}$ & 91.47$_{100\%}$ & 69.69$_{100\%}$ & 97.91$_{100\%}$ & 67.50$_{100\%}$ & 98.88$_{100\%}$ & 85.47$_{100\%}$ \\
\midrule
JAN & 77.74 & 77.64 & 64.48 & 91.68 & 92.60 & 65.10 & 78.21 \\
IWJAN & \textbf{84.62}$_{100\%}$ & \textbf{83.28}$_{100\%}$ & \textbf{65.30}$_{80\%}$ & \textbf{96.30}$_{100\%}$ & \textbf{98.80}$_{100\%}$ & \textbf{67.38}$_{100\%}$ & \textbf{82.61}$_{97\%}$ \\
IWJAN-O & 88.42$_{100\%}$ & 89.44$_{100\%}$ & 72.06$_{100\%}$ & 97.26$_{100\%}$ & 98.96$_{100\%}$ & 71.30$_{100\%}$ & 86.24$_{100\%}$ \\
\bottomrule
\end{tabular}
\end{sc}
\end{small}
\end{center}
\vskip -0.1in
\end{table*}

\begin{table*}[ht]
\caption{Results on the Office-Home dataset.}
\label{tab:home_full}
\vskip 0.15in
\begin{center}
\begin{small}
\begin{sc}
\begin{tabular}{lcccccc|c}
\toprule
Method & A $\rightarrow$ C & A $\rightarrow$ P & A $\rightarrow$ R & C $\rightarrow$ A & C $\rightarrow$ P & C $\rightarrow$ R \\
\midrule
No DA & 41.02 & 62.97 & 71.26 & 48.66 & 58.86 & 60.91 \\
\midrule
DANN & 46.03 & 62.23 & 70.57 & 49.06 & 63.05 & 64.14 \\
IWDAN & \textbf{48.65}$_{100\%}$ & \textbf{69.19}$_{100\%}$ & \textbf{73.60}$_{100\%}$ & \textbf{53.59}$_{100\%}$ & \textbf{66.25}$_{100\%}$ & \textbf{66.09}$_{100\%}$  \\
IWDAN-O & 50.19$_{100\%}$ & 70.53$_{100\%}$ & 75.44$_{100\%}$ & 56.69$_{100\%}$ & 67.40$_{100\%}$ & 67.98$_{100\%}$ \\
\midrule
CDAN & 49.00 & 69.23 & 74.55 & 54.46 & 68.23 & 68.9 \\
IWCDAN & \textbf{49.81}$_{100\%}$ & \textbf{73.41}$_{100\%}$ & \textbf{77.56}$_{100\%}$ & \textbf{56.5}$_{100\%}$ & \textbf{69.64}$_{80\%}$ & \textbf{70.33}$_{100\%}$ \\
IWCDAN-O & 52.31$_{100\%}$ & 74.54$_{100\%}$ & 78.46$_{100\%}$ & 60.33$_{100\%}$ & 70.78$_{100\%}$ & 71.47$_{100\%}$ \\
\midrule
JAN & \textbf{41.64} & 67.20 & 73.12 & 51.02 & 62.52 & 64.46 \\
IWJAN & 41.12$_{0\%}$ & \textbf{67.56}$_{80\%}$ & \textbf{73.14}$_{60\%}$ & \textbf{51.70}$_{100\%}$ & \textbf{63.42}$_{100\%}$ & \textbf{65.22}$_{100\%}$ \\
IWJAN-O & 41.88$_{80\%}$ & 68.72$_{100\%}$ & 73.62$_{100\%}$ & 53.04$_{100\%}$ & 63.88$_{100\%}$ & 66.48$_{100\%}$ \\
\midrule
\midrule
Method & P $\rightarrow$ A & P $\rightarrow$ C & P $\rightarrow$ R & R $\rightarrow$ A & R $\rightarrow$ C & R $\rightarrow$ P & Avg. \\
\midrule
No DA & 47.1 & 35.94 & 68.27 & 61.79 & 44.42 & 75.5 & 56.39\\
\midrule
DANN  & 48.29 & 44.06 & 72.62 & 63.81 & 53.93 & 77.64 & 59.62\\
IWDAN & \textbf{52.81}$_{100\%}$ & \textbf{46.24}$_{80\%}$ & \textbf{73.97}$_{100\%}$ & \textbf{64.90}$_{100\%}$ & \textbf{54.02}$_{80\%}$ & \textbf{77.96}$_{100\%}$ & \textbf{62.27}$_{97\%}$ \\
IWDAN-O & 59.33$_{100\%}$ & 48.28$_{100\%}$ & 76.37$_{100\%}$ & 69.42$_{100\%}$ & 56.09$_{100\%}$ & 78.45$_{100\%}$ & 64.68$_{100\%}$\\
\midrule
CDAN & 56.77 & \textbf{48.8} & 76.83 & \textbf{71.27} & \textbf{55.72} & \textbf{81.27} & 64.59\\
IWCDAN & \textbf{58.99}$_{100\%}$ & 48.41$_{0\%}$ & \textbf{77.94}$_{100\%}$ & 69.48$_{0\%}$ & 54.73$_{0\%}$ & 81.07$_{60\%}$ & \textbf{65.66}$_{70\%}$ \\
IWCDAN-O & 62.60$_{100\%}$ & 50.73$_{100\%}$ & 78.88$_{100\%}$ & 72.44$_{100\%}$ & 57.79$_{100\%}$ & 81.31$_{80\%}$ & 67.64$_{98\%}$\\
\midrule
JAN & 54.5 & 40.36 & \textbf{73.10} & \textbf{64.54} & \textbf{45.98} & \textbf{76.58} & 59.59 \\
IWJAN & \textbf{55.26}$_{80\%}$ & \textbf{40.38}$_{60\%}$ & 73.08$_{80\%}$ & 64.40$_{60\%}$ & 45.68$_{0\%}$ & 76.36$_{40\%}$ & \textbf{59.78}$_{63\%}$ \\
IWJAN-O & 57.78$_{100\%}$ & 41.32$_{100\%}$ & 73.66$_{100\%}$ & 65.40$_{100\%}$ & 46.68$_{100\%}$ & 76.36$_{20\%}$ & 60.73$_{92\%}$ \\
\bottomrule
\end{tabular}
\end{sc}
\end{small}
\end{center}
\vskip -0.1in
\end{table*}

\begin{table*}[ht]
\caption{Results on the subsampled Office-Home dataset.}
\label{tab:home_sub}
\vskip 0.15in
\begin{center}
\begin{small}
\begin{sc}
\begin{tabular}{lcccccccccccc|c}
\toprule
Method & A $\rightarrow$ C & A $\rightarrow$ P & A $\rightarrow$ R & C $\rightarrow$ A & C $\rightarrow$ P & C $\rightarrow$ R \\
\midrule
No DA & 35.70 & 54.72 & 62.61 & 43.71 & 52.54 & 56.62 \\
\midrule
DANN & 36.14 & 54.16 & 61.72 & 44.33 & 52.56 & 56.37 \\
IWDAN & \textbf{39.81}$_{100\%}$ & \textbf{63.01}$_{100\%}$ & \textbf{68.67}$_{100\%}$ & \textbf{47.39}$_{100\%}$ & \textbf{61.05}$_{100\%}$ & \textbf{60.44}$_{100\%}$ \\
IWDAN-O & 42.79$_{100\%}$ & 66.22$_{100\%}$ & 71.40$_{100\%}$ & 53.39$_{100\%}$ & 61.47$_{100\%}$ & 64.97$_{100\%}$ \\
\midrule
CDAN & 38.90 & 56.80 & 64.77 & 48.02 & 60.07 & 61.17 \\
IWCDAN & \textbf{42.96}$_{100\%}$ & \textbf{65.01}$_{100\%}$ & \textbf{71.34}$_{100\%}$ & \textbf{52.89}$_{100\%}$ & \textbf{64.65}$_{100\%}$ & \textbf{66.48}$_{100\%}$ \\
IWCDAN-O & 45.76$_{100\%}$ & 68.61$_{100\%}$ & 73.18$_{100\%}$ & 56.88$_{100\%}$ & 66.61$_{100\%}$ & 68.48$_{100\%}$ \\
\midrule
JAN & 34.52 & 56.86 & 64.54 & 46.18 & 56.84 & 59.06 \\
IWJAN & \textbf{36.24}$_{100\%}$ & \textbf{61.00}$_{100\%}$ & \textbf{66.34}$_{100\%}$ & \textbf{48.66}$_{100\%}$ & \textbf{59.92}$_{100\%}$ & \textbf{61.88}$_{100\%}$ \\
IWJAN-O & 37.46$_{100\%}$ & 62.68$_{100\%}$ & 66.88$_{100\%}$ & 49.82$_{100\%}$ & 60.22$_{100\%}$ & 62.54$_{100\%}$ \\
\midrule
\midrule
Method & P $\rightarrow$ A & P $\rightarrow$ C & P $\rightarrow$ R & R $\rightarrow$ A & R $\rightarrow$ C & R $\rightarrow$ P & Avg. \\
\midrule
No DA & 44.29 & 33.05 & 65.20 & 57.12 & 40.46 & 70.0 \\
\midrule
DANN & 44.58 & 37.14 & 65.21 & 56.70 & 43.16 & 69.86 & 51.83 \\
IWDAN & \textbf{50.44}$_{100\%}$ & \textbf{41.63}$_{100\%}$ & \textbf{72.46}$_{100\%}$ & \textbf{61.00}$_{100\%}$ & \textbf{49.40}$_{100\%}$ & \textbf{76.07}$_{100\%}$ & \textbf{57.61}$_{100\%}$ \\
IWDAN-O & 56.05$_{100\%}$ & 43.39$_{100\%}$ & 74.87$_{100\%}$ & 66.73$_{100\%}$ & 51.72$_{100\%}$ & 77.46$_{100\%}$ & 60.87$_{100\%}$ \\
\midrule
CDAN & 49.65 & 41.36 & 70.24 & 62.35 & 46.98 & 74.69 & 56.25 \\
IWCDAN & \textbf{54.87}$_{100\%}$ & \textbf{44.80}$_{100\%}$ & \textbf{75.91}$_{100\%}$ & \textbf{67.02}$_{100\%}$ & \textbf{50.45}$_{100\%}$ & \textbf{78.55}$_{100\%}$ & \textbf{61.24}$_{100\%}$ \\
IWCDAN-O & 59.63$_{100\%}$ & 46.98$_{100\%}$ & 77.54$_{100\%}$ & 69.24$_{100\%}$ & 53.77$_{100\%}$ & 78.11$_{100\%}$ & 63.73$_{100\%}$ \\
\midrule
JAN & 50.64 & 37.24 & 69.98 & 58.72 & 40.64 & 72.00 & 53.94 \\
IWJAN & \textbf{52.92}$_{100\%}$ & \textbf{37.68}$_{100\%}$ & \textbf{70.88}$_{100\%}$ & \textbf{60.32}$_{100\%}$ & \textbf{41.54}$_{100\%}$ & \textbf{73.26}$_{100\%}$ & \textbf{55.89}$_{100\%}$ \\
IWJAN-O & 56.54$_{100\%}$ & 39.66$_{100\%}$ & 71.78$_{100\%}$ & 62.36$_{100\%}$ & 44.56$_{100\%}$ & 73.76$_{100\%}$ & 57.36$_{100\%}$ \\
\bottomrule
\end{tabular}
\end{sc}
\end{small}
\end{center}
\vskip -0.1in
\end{table*}

\subsection{Jensen-Shannon divergence of the original and subsampled domain adaptation datasets}\label{sub:jsd}

Tables~\ref{tab:jsd_digits}, \ref{tab:jsd_office} and \ref{tab:jsd_home} show $\jsd(\dist_S(Z)||\dist_T(Z))$ for our four datasets and their subsampled versions, rows correspond to the source domain, and columns to the target one. We recall that subsampling simply consists in taking $30\%$ of the first half of the classes in the source domain (which explains why $\jsd(\dist_S(Z)||\dist_T(Z))$ is not symmetric for the subsampled datasets).

\begin{table*}[ht]
\caption{Jensen-Shannon divergence between the label distributions of the Digits and Visda tasks.}
\label{tab:jsd_digits}
\vskip 0.15in
\centering
\begin{small}
\begin{sc}
\subfloat[Full Dataset\label{tab:jsd_digits_full}]{
\centering
\begin{tabular}{lccccccc}
\toprule
 & MNIST & USPS &  Real \\
\midrule
MNIST  & 0  & $6.64\mathrm{e}{-3}$ & - \\
USPS  & $6.64\mathrm{e}{-3}$ & 0 & - \\
Synth. & -                    & -                    & $2.61\mathrm{e}{-2}$ \\
\bottomrule
\end{tabular}
}\hspace{2cm}
\subfloat[Subsampled\label{tab:jsd_digits_sub}]{
\centering
\begin{tabular}{lccccccc}
\toprule
 & MNIST & USPS &  Real \\
\midrule
MNIST  & 0  & $6.52\mathrm{e}{-2}$ & - \\
USPS  & $2.75\mathrm{e}{-2}$ & 0 & - \\
Synth. & -                    & -                    & $6.81\mathrm{e}{-2}$ \\
\bottomrule
\end{tabular}
}
\end{sc}
\end{small}
\vskip -0.1in
\end{table*}

\begin{table*}[ht]
\caption{Jensen-Shannon divergence between the label distributions of the Office-31 tasks.}
\label{tab:jsd_office}
\vskip 0.15in
\centering
\begin{small}
\begin{sc}
\subfloat[Full Dataset\label{tab:jsd_office_full}]{
\centering
\begin{tabular}{lccccccc}
\toprule
 & Amazon & DSLR & Webcam \\
\midrule
Amazon & 0 & $1.76\mathrm{e}{-2}$ & $9.52\mathrm{e}{-3}$ \\
DSLR   & $1.76\mathrm{e}{-2}$ & 0 & $2.11\mathrm{e}{-2}$ \\
Webcam   & $9.52\mathrm{e}{-3}$ & $2.11\mathrm{e}{-2}$ & 0 \\
\bottomrule
\end{tabular}
}\hspace{2cm}
\subfloat[Subsampled\label{tab:jsd_office_sub}]{
\centering
\begin{tabular}{lccccccc}
\toprule
 & Amazon & DSLR & Webcam \\
\midrule
Amazon & 0 & $6.25\mathrm{e}{-2}$ & $4.61\mathrm{e}{-2}$ \\
DSLR & $5.44\mathrm{e}{-2}$ & 0 & $5.67\mathrm{e}{-2}$ \\
Webcam   & $5.15\mathrm{e}{-2}$ & $7.05\mathrm{e}{-2}$ & 0 \\
\bottomrule
\end{tabular}
}
\end{sc}
\end{small}
\vskip -0.1in
\end{table*}

\begin{table*}[ht]
\caption{Jensen-Shannon divergence between the label distributions of the Office-Home tasks.}
\label{tab:jsd_home}
\vskip 0.15in
\centering
\begin{sc}
\subfloat[Full Dataset\label{tab:jsd_home_full}]{
\centering
\begin{tabular}{lccccccc}
\toprule
 & Art & Clipart & Product & Real World \\
\midrule
Art  & 0 & $3.85\mathrm{e}{-2}$ & $4.49\mathrm{e}{-2}$ & $2.40\mathrm{e}{-2}$ \\
Clipart  & $3.85\mathrm{e}{-2}$ & 0 & $2.33\mathrm{e}{-2}$ & $2.14\mathrm{e}{-2}$ \\
Product & $4.49\mathrm{e}{-2}$ & $2.33\mathrm{e}{-2}$ & 0 & $1.61\mathrm{e}{-2}$ \\
Real World & $2.40\mathrm{e}{-2}$ & $2.14\mathrm{e}{-2}$ & $1.61\mathrm{e}{-2}$ & 0 \\
\bottomrule
\end{tabular}
}\hspace{1cm}
\subfloat[Subsampled\label{tab:jsd_home_sub}]{
\centering
\begin{tabular}{lccccccc}
\toprule
 & Art  & Clipart & Product & Real World \\
\midrule
Art        & 0 & $8.41\mathrm{e}{-2}$ & $8.86\mathrm{e}{-2}$ & $6.69\mathrm{e}{-2}$ \\
Clipart    & $7.07\mathrm{e}{-2}$ & 0 & $5.86\mathrm{e}{-2}$ & $5.68\mathrm{e}{-2}$ \\
Product    & $7.85\mathrm{e}{-2}$ & $6.24\mathrm{e}{-2}$ & 0 & $5.33\mathrm{e}{-2}$ \\
Real World & $6.09\mathrm{e}{-2}$ & $6.52\mathrm{e}{-2}$ & $5.77\mathrm{e}{-2}$ & 0 \\
\bottomrule
\end{tabular}
}
\end{sc}
\end{table*}

\subsection{Losses}\label{sub:losses}

\subsubsection{DANN}

For batches of data $(x^i_S,y^i_S)$ and $(x^i_T)$ of size $s$, the $\dann$ losses are:
\begin{align}
    \mathcal{L}_{DA}(x^i_S,y^i_S,x_T^i; \theta, \psi) &= 
    \hspace{0.1cm}- \frac{1}{s} \displaystyle{\sum_{i = 1}^s} \log(d_\psi(g_\theta(x_S^i)))
    + \log(1 - d_\psi(g_\theta(x_T^i))), \label{eq:da_loss} \\
    \mathcal{L}_{C}(x^i_S,y^i_S; \theta, \phi) &= -\frac{1}{s} \displaystyle{\sum_{i = 1}^s} \log(h_\phi(g_\theta(x_S^i)_{y^i_S})). \label{eq:da_ce_loss}
\end{align}

\subsubsection{CDAN}\label{subsub:losses_cdan}

Similarly, the $\cdan$ losses are:
\begin{align}
    \mathcal{L}_{DA}(x^i_S,y^i_S,x_T^i; \theta, \psi) &= 
    \hspace{0.1cm}- \frac{1}{s} \displaystyle{\sum_{i = 1}^s} \log(d_\psi(h_\phi(g_\theta(x_S^i)) \otimes g_\theta(x_S^i))) \\
    & \hspace{2.5cm} + \log(1 - d_\psi(h_\phi(g_\theta(x_T^i)) \otimes g_\theta(x_T^i))), \label{eq:cda_loss} \\
    \mathcal{L}_{C}(x^i_S,y^i_S; \theta, \phi) &= -\frac{1}{s} \displaystyle{\sum_{i = 1}^s} \log(h_\phi(g_\theta(x_S^i)_{y^i_S})), \label{eq:cda_ce_loss}
\end{align}
where $h_\phi(g_\theta(x_S^i)) \otimes g_\theta(x_S^i) \defeq (h_1(g(x_S^i)) g(x_S^i),\dots,h_k(g(x_S^i)) g(x_S^i))$ and $h_1(g(x_S^i))$ is the $i$-th element of vector $h(g(x_S^i))$.

$\cdan$ is particularly well-suited for conditional alignment. As described in Section~\ref{sec:preliminary}, the $\cdan$ discriminator seeks to match $\dist_S(\Ypred \otimes Z)$ with $\dist_T(\Ypred \otimes Z)$. This objective is very aligned with $\glsa$: let us first assume for argument's sake that $\Ypred$ is a perfect classifier on both domains. For any sample $(x,y)$, $\hat{y} \otimes z$ is thus a matrix of $0$s except on the $y$-th row, which contains $z$. When label distributions match, the effect of fooling the discriminator will result in representations such that the matrices $\Ypred \otimes Z$ are equal on the source and target domains. In other words, the model is such that $Z \mid Y$ match: it verifies $\glsa$ (see Th.~\ref{thm:sufficient_condition} below with $\ww = 1$). On the other hand, if the label distributions differ, fooling the discriminator actually requires mislabelling certain samples (a fact quantified in Th.~\ref{thm:lower_bound}).

\subsubsection{JAN}

The $\jan$ losses~\citep{long2017deep} are :
\begin{align}
    \mathcal{L}_{DA}(x^i_S,y^i_S,x_T^i; \theta, \psi) &= 
    \hspace{0.1cm}- \frac{1}{s^2} \displaystyle{\sum_{i,j = 1}^s} k(x^i_S,x^j_S) - \frac{1}{s^2} \displaystyle{\sum_{i,j = 1}^s} k(x^i_T,x^j_T) + \frac{2}{s^2} \displaystyle{\sum_{i,j = 1}^s} k(x^i_S,x^j_T) \label{eq:jan_loss}
    \\
    \mathcal{L}_{C}(x^i_S,y^i_S; \theta, \phi) &= -\frac{1}{s} \displaystyle{\sum_{i = 1}^s} \log(h_\phi(g_\theta(x_S^i)_{y^i_S})), \label{eq:jan_ce_loss}
\end{align}
where $k$ corresponds to the kernel of the RKHS $\HH$ used to measure the discrepancy between distributions. Exactly as in~\citet{long2017deep}, it is the product of kernels on various layers of the network $k(x^i_S,x^j_S) = \prod_{l \in \mathcal{L}} k^l(x^i_S,x^j_S)$. Each individual kernel $k^l$ is computed as the dot-product between two transformations of the representation: $k^l(x^i_S,x^j_S) = \langle d^l_\psi(g^l_\theta(x_S^i)), d^l_\psi(g^l_\theta(x_S^j)) \rangle$ (in this case, $d^l_\psi$ outputs vectors in a high-dimensional space). See Section~\ref{sub:imp} for more details.

The $\iwjan$ losses are:
\begin{align}
    \mathcal{L}^{\ww}_{DA}(x^i_S,y^i_S,x_T^i; \theta, \psi) &= 
    \hspace{0.1cm}- \frac{1}{s^2} \displaystyle{\sum_{i,j = 1}^s} \ww_{y^i_S} \ww_{y^j_S} k(x^i_S,x^j_S) - \frac{1}{s^2} \displaystyle{\sum_{i,j = 1}^s} k(x^i_T,x^j_T) + \frac{2}{s^2} \displaystyle{\sum_{i,j = 1}^s} \ww_{y^i_S} k(x^i_S,x^j_T) \label{eq:iwjan_loss}
    \\
    \mathcal{L}_{C}^{\ww}(x^i_S,y^i_S; \theta, \phi) &=
    -\frac{1}{s} \displaystyle{\sum_{i = 1}^s} \frac{\ww_{y^i_S}}{k \dist_S(Y = y)}\log(h_\phi(g_\theta(x_S^i))_{y^i_S}).
    \label{eq:iwjan_ce_loss}
\end{align}

\subsection{Generation of domain adaptation tasks with varying \texorpdfstring{$\jsd(\dist_S(Z)~\|~\dist_T(Z))$}{alpha}}
\label{sub:jsd_generation}

We consider the MNIST $\rightarrow$ USPS task and generate a set $\mathcal{V}$ of $50$ vectors in $[0.1, 1]^{10}$. Each vector corresponds to the fraction of each class to be trained on, either in the source or the target domain (to assess the impact of both). The left bound is chosen as $0.1$ to ensure that classes all contain some samples. 

This methodology creates $100$ domain adaptation tasks, $50$ for \textit{subsampled}-MNIST $\rightarrow$ USPS and $50$ for MNIST $\rightarrow$ \textit{subsampled}-USPS, with Jensen-Shannon divergences varying from $6.1\mathrm{e}{-3}$ to $9.53\mathrm{e}{-2}$\footnote{We manually rejected some samples to guarantee a rather uniform set of divergences.}. They are then used to evaluate our algorithms, see Section \ref{sec:imp} and Figures \ref{fig:cloud} and \ref{fig:cloud_perf}. They show the performance of the 6 algorithms we consider. We see the sharp decrease in performance of the base versions $\dann$ and $\cdan$. Comparatively, our importance-weighted algorithms maintain good performance even for large divergences between the marginal label distributions.

\subsection{Implementation details}\label{sub:imp}

All the values reported below are the default ones in the implementations of DANN, CDAN and JAN released with the respective papers (see links to the github repos in the footnotes). We did not perform any search on them, assuming they had already been optimized by the authors of those papers. To ensure a fair comparison and showcase the simplicity of our approach, we simply plugged the weight estimation on top of those baselines and used their original hyperparameters.

For MNIST and USPS, the architecture is akin to LeNet~\citep{726791}, with two convolutional layers, ReLU and MaxPooling, followed by two fully connected layers. The representation is also taken as the last hidden layer, and has 500 neurons. The optimizer for those tasks is SGD with a learning rate of $0.02$, annealed by $0.5$ every five training epochs for \mbox{M $\rightarrow$ U} and $6$ for \mbox{U $\rightarrow$ M}. The weight decay is also $5\mathrm{e}{-4}$ and the momentum $0.9$.

For the Office and Visda experiments with $\iwdan$ and $\iwcdan$, we train a ResNet-50, optimized using SGD with momentum. The weight decay is also $5\mathrm{e}{-4}$ and the momentum $0.9$. The learning rate is $3\mathrm{e}{-4}$ for the Office-31 tasks \mbox{A $\rightarrow$ D} and \mbox{D $\rightarrow$ W}, $1\mathrm{e}{-3}$ otherwise (default learning rates from the CDAN implementation\footnote{\url{https://github.com/thuml/CDAN/tree/master/pytorch}}).

For the $\iwjan$ experiments, we use the default implementation of Xlearn codebase\footnote{\url{https://github.com/thuml/Xlearn/tree/master/pytorch}} and simply add the weigths estimation and reweighted objectives to it, as described in Section~\ref{sub:losses}. Parameters, configuration and networks remain the same. 

Finally, for the Office experiments, we update the importance weights $\ww$ every 15 passes on the dataset (in order to improve their estimation on small datasets). On Digits and Visda, the importance weights are updated every pass on the source dataset. Here too, fine-tuning that value might lead to a better estimation of $\ww$ and help bridge the gap with the oracle versions of the algorithms.

We use the cvxopt package\footnote{\url{http://cvxopt.org/}} to solve the quadratic programm \ref{qp}.

We trained our models on single-GPU machines (P40s and P100s). The runtime of our algorithms is undistinguishable from the the runtime of their base versions.

\begin{figure*}[t]
	\centering
	\subfloat[Performance of $\dann$, $\iwdan$ and $\iwdano$.]{
		\includegraphics[width=0.48\linewidth]{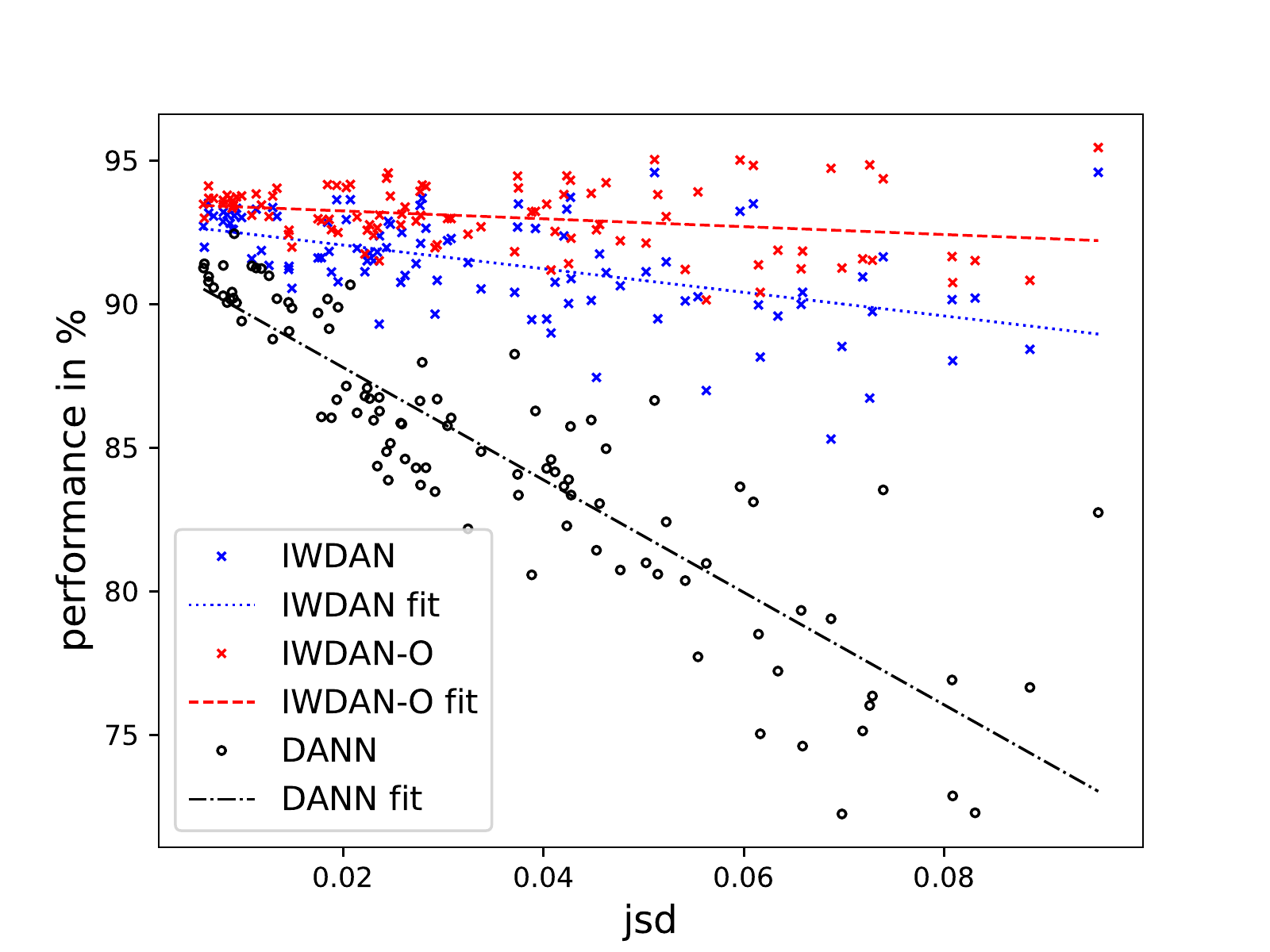}
		\label{fig:cloud_perf_iwdan}
	}
	\subfloat[Performance of $\cdan$, $\cdan$ and $\iwcdan$.]{
		\includegraphics[width=0.48\linewidth]{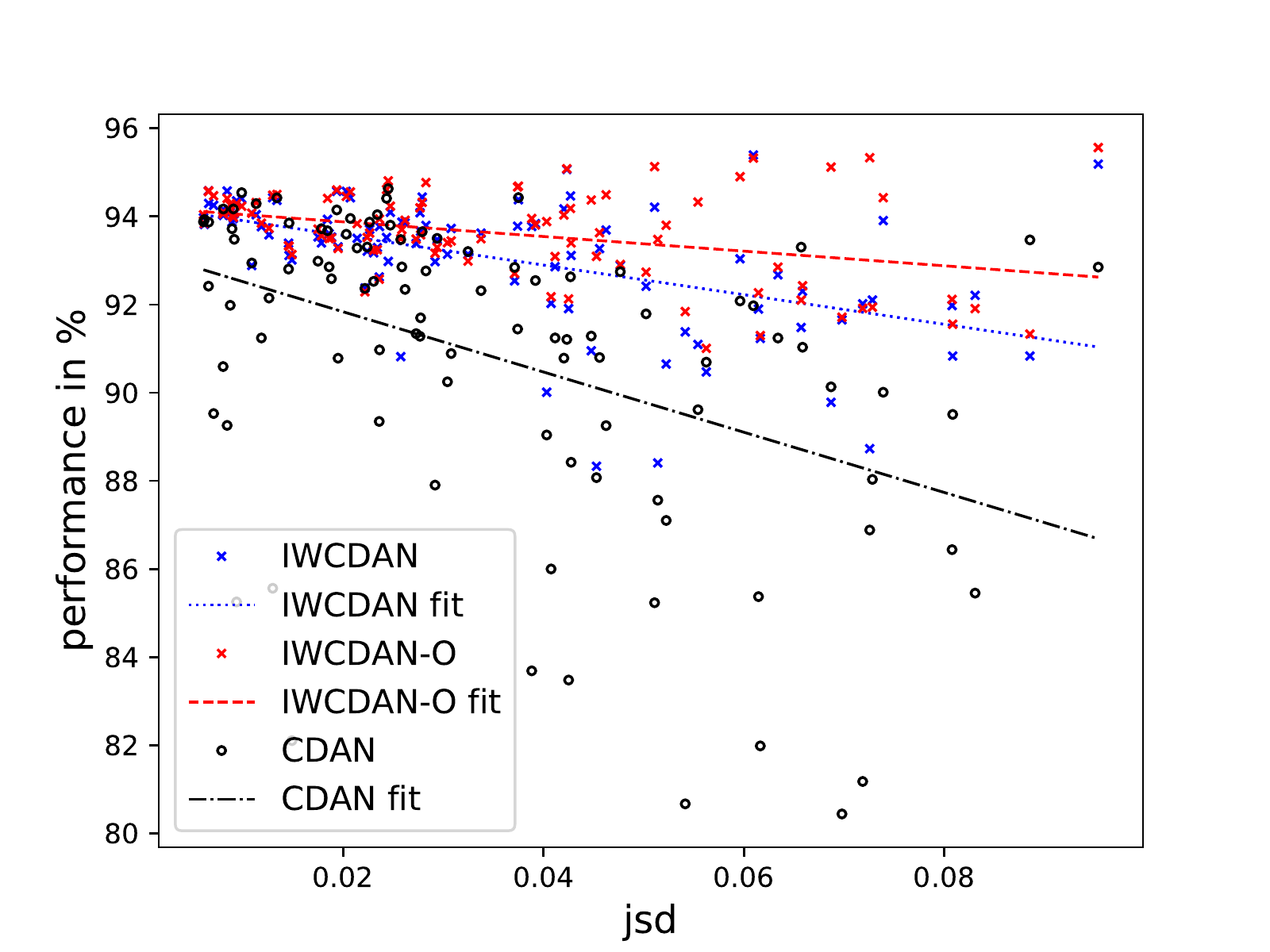}
		\label{fig:cloud_perf_iwcdan}
	}
	\caption{Performance in \% of our algorithms and their base versions. The $x$-axis represents \mbox{$\jsd(\domain_S^Y, \domain_T^Y)$}, the Jensen-Shannon distance between label distributions. Lines represent linear fits to the data. For both sets of algorithms, the larger the jsd, the larger the improvement.}
	\label{fig:cloud_perf}
\end{figure*}


\subsection{Weight Estimation}\label{sub:estimation}

We estimate importance weights using Lemma~\ref{lem:iw_equation}, which relies on the $\glsa$  assumption. However, there is no guarantee that $\glsa$ is verified at any point during training, so the exact dynamics of $\ww$ are unclear. Below we discuss those dynamics and provide some intuition about them.

In Fig.~\ref{fig:distance_app}, we plot the Euclidian distance between the moving average of weights estimated using the equation $\ww = \textbf{C}^{-1} \boldsymbol\mu$ and the true weights (note that this can be done for any algorithm). As can be seen in the figure, the distance between the estimated and true weights is highly correlated with the performance of the algorithm (see Fig.\ref{fig:acc_distance_app}). In particular, we see that the estimations for $\iwdan$ is more accurate than for $\dann$. The estimation for $\dann$ exhibits an interesting shape, improving at first, and then getting worse. At the same time, the estimation for $\iwdan$ improves monotonously. The weights for $\iwdano$ get very close to the true weights which is in line with our theoretical results: $\iwdano$ gets close to zero error on the target error, Th.~\ref{thm:sufficient_condition} thus guarantees that $\glsa$ is verified, which in turns implies that the weight estimation is accurate (Lemma~\ref{lem:iw_equation}). Finally, without domain adaptation, the estimation is very poor. The following two lemmas shed some light on the phenomena observed for $\dann$ and $\iwdan$:

\weightsconvergence*
\begin{proof}
If $\err_S(h\circ g) = 0$, then the confusion matrix $\textbf{C}$ is diagonal and its $y$-th line is $\dist_S(Y = y)$. Additionally, if $\jsd(\dist_S^{\tilde{\ww}}(Z), \dist_T(Z)) = 0$, then from a straightforward extension of Eq.~\ref{eq:generalized_DPI}, we have $\jsd(\dist_S^{\tilde{\ww}}(\hat{Y}), \dist_T(\hat{Y})) = 0$. In other words, the distribution of predictions on the source and target domains match, \textit{i.e.} $\boldsymbol\mu_y = \dist_T(\hat{Y} = y) = \displaystyle{\sum_{y'}} \tilde{\ww}_{y'} \dist_S(\hat{Y} = y, Y = y') = \tilde{\ww}_{y} \dist_S(Y = y), \forall y$ (where the last equality comes from $\err_S(h\circ g) = 0$). Finally, we get that $\ww = \textbf{C}^{-1} \boldsymbol\mu = \tilde{\ww}$.
\end{proof}

In particular, applying this lemma to $\dann$ (\textit{i.e.} with $\tilde{\ww}_{y'} = \textbf{1}$) suggests that at convergence, the estimated weights should tend to $\textbf{1}$. Empirically, Fig.~\ref{fig:distance_app} shows that as the marginals get matched, the estimation for $\dann$ does get closer to $\textbf{1}$ ($\textbf{1}$ corresponds to a distance of $2.16$)\footnote{It does not reach it as the learning rate is decayed to $0$.}. We now attempt to provide some intuition on the behavior of $\iwdan$, with the following lemma:

\begin{lemma}\label{lem:monotonous}
If $\err_S(h\circ g) = 0$ and if for a given $y$:
\begin{equation}
    \min(\tilde{\ww_y} \dist_S(Y = y), \dist_T(Y = y)) \leq \boldsymbol\mu_y \leq \max(\tilde{\ww_y} \dist_S(Y = y), \dist_T(Y = y)), \label{eq:monotonous}
\end{equation}
then, letting $\ww = \textbf{C}^{-1} \boldsymbol\mu$ be the estimated weight:
\begin{equation*}
    |\ww_y - \ww_y^*| \leq |\tilde{\ww_y} - \ww_y^*|.
\end{equation*}
\end{lemma}

Applying this lemma to $\tilde{\ww_y} = \ww_t$, and assuming that~\eqref{eq:monotonous} holds for all the classes $y$ (we discuss what the assumption implies below), we get that:
\begin{equation}
    \|\ww_{t+1} - \ww_y^*\| \leq \|\ww_t - \ww_y^*\|,
\end{equation}
or in other words, the estimation improves monotonously. Combining this with Lemma~\ref{lem:monotonous} suggests an explanation for the shape of the $\iwdan$ estimated weights on Fig.~\ref{fig:distance_app}: the monotonous improvement of the estimation is counter-balanced by the matching of weighted marginals which, when reached, makes $\ww_t$ constant (Lemma~\ref{lem:convergence} applied to $\tilde{\ww} = \ww_t$). However, we wish to emphasize that the exact dynamics of $\ww$ are complex, and we do not claim understand them fully. In all likelihood, they are the by-product of regularity in the data, properties of deep neural networks and their interaction with stochastic gradient descent. Additionally, the dynamics are also inherently linked to the success of domain adaptation, which to this day remains an open problem.

As a matter of fact, assumption~\eqref{eq:monotonous} itself relates to successful domain adaptation. Setting aside $\tilde{\ww}$, which simply corresponds to a class reweighting of the source domain, \eqref{eq:monotonous} states that predictions on the target domain fall between a successful prediction (corresponding to $\dist_T(Y = y)$) and the prediction of a model with matched marginals (corresponding to $\dist_S(Y = y)$). In other words, we assume that the model is naturally in between successful domain adaptation and successful marginal matching. Empirically, we observed that it holds true for most classes (with $\tilde{\ww} = \tilde{\ww}_t$ for $\iwdan$ and with $\tilde{\ww} = \textbf{1}$ for $\dann$), but not all early in training\footnote{In particular at initialization, one class usually dominates the others.}.

To conclude this section, we prove Lemma~\ref{lem:monotonous}.
\begin{proof}
From $\err_S(h\circ g) = 0$, we know that $\textbf{C}$ is diagonal and that its $y$-th line is $\dist_S(Y = y)$. This gives us: $\ww_y = (\textbf{C}^{-1} \boldsymbol\mu)_y = \frac{\boldsymbol\mu_y}{\dist_S(Y = y)}$. Hence:

\begin{align*}
    & \min(\tilde{\ww_y} \dist_S(Y = y), \dist_T(Y = y)) \leq \boldsymbol\mu_y \leq \max(\tilde{\ww_y} \dist_S(Y = y), \dist_T(Y = y)) \\
    \Longleftrightarrow \quad & \frac{\min(\tilde{\ww_y} \dist_S(Y = y), \dist_T(Y = y))}{\dist_S(Y = y)} \leq \frac{\boldsymbol\mu_y}{\dist_S(Y = y)} \leq \frac{\max(\tilde{\ww_y} \dist_S(Y = y), \dist_T(Y = y))}{\dist_S(Y = y)} \\
    \Longleftrightarrow \quad & \min(\tilde{\ww_y}, \ww_y^*) \leq \ww_y \leq \max(\tilde{\ww_y}, \ww_y^*) \\
    \Longleftrightarrow \quad & \min(\tilde{\ww_y}, \ww_y^*) - \ww_y^* \leq \ww_y - \ww_y^* \leq \max(\tilde{\ww_y}, \ww_y^*) - \ww_y^* \\
    \Longleftrightarrow \quad & |\ww_y - \ww_y^*| \leq |\tilde{\ww_y} - \ww_y^*|,
\end{align*}
which conludes the proof.
\end{proof}

\begin{figure*}[ht]
	\centering
	\subfloat[Transfer accuracy during training.]{
		\includegraphics[width=0.48\linewidth]{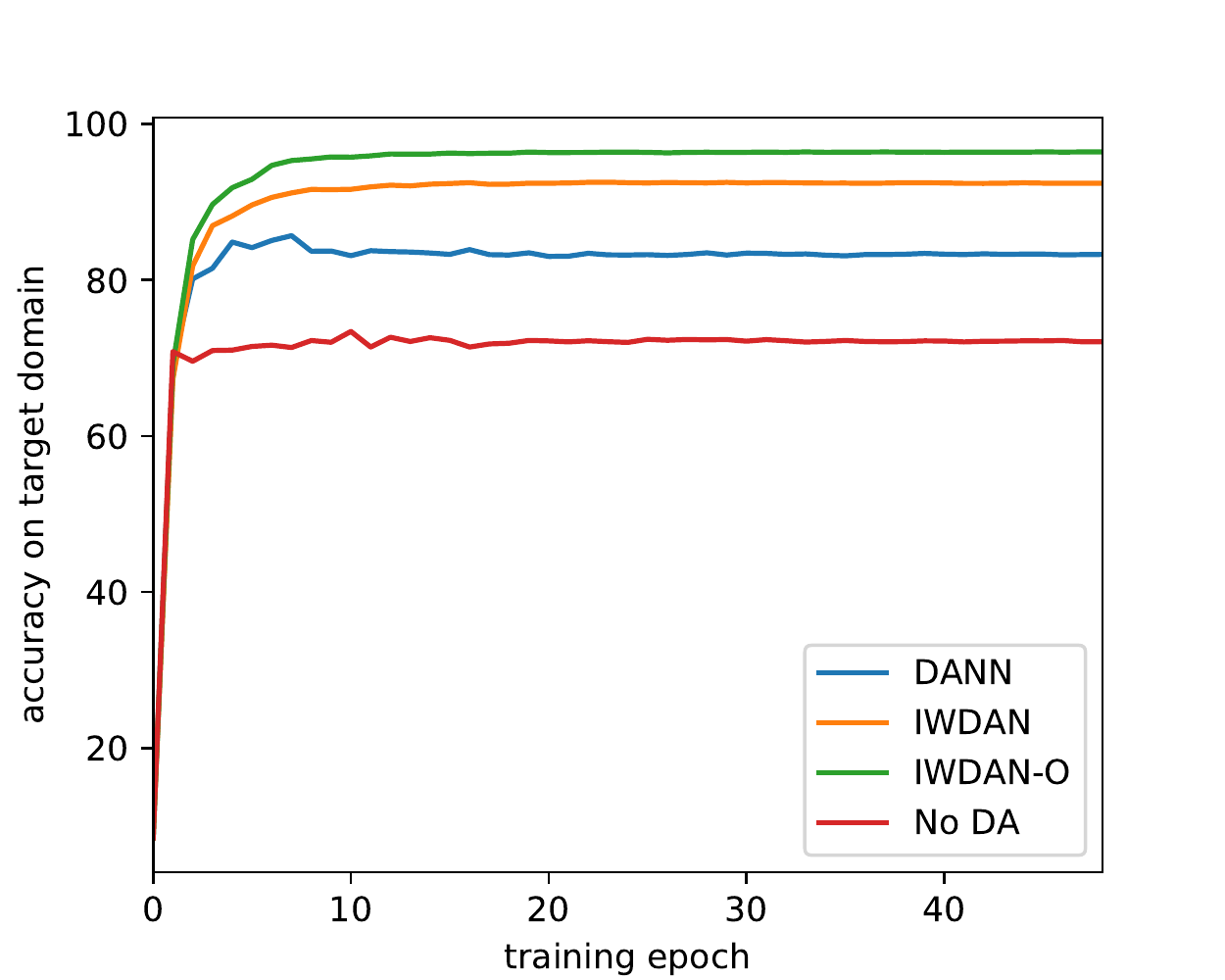}
		\label{fig:accuracy_app}
	}
	\subfloat[Distance to true weights during training.]{
		\includegraphics[width=0.48\linewidth]{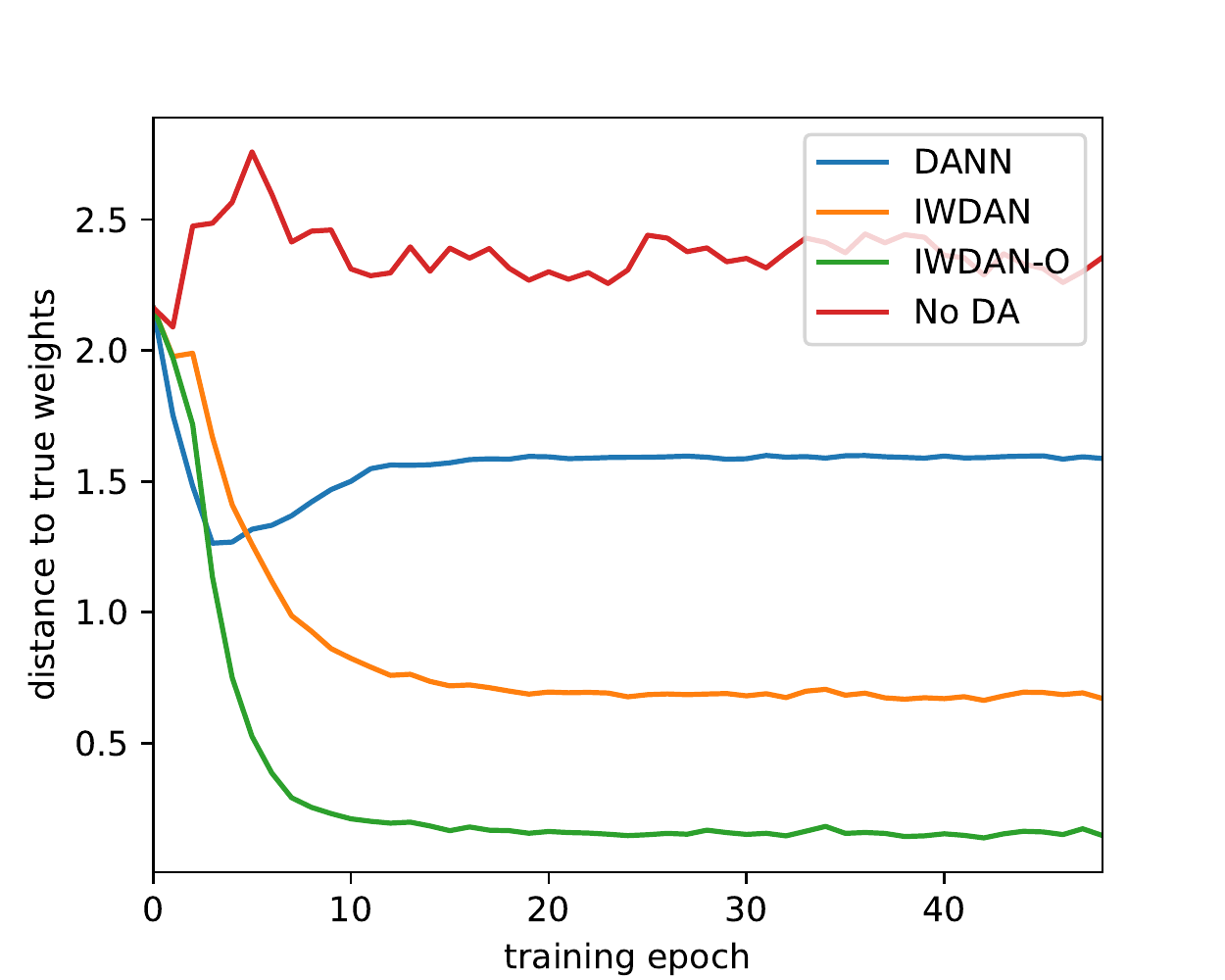}
		\label{fig:distance_app}
	}
	\caption{\textit{Left} Accuracy of various algorithms during training. \textit{Right} Euclidian distance between the weights estimated using Lemma~\ref{lem:iw_equation} and the true weights. Those plots correspond to averages over 5 seeds.}
	\label{fig:acc_distance_app}
\end{figure*}

\subsection{Per-class predictions and estimated weights}\label{sub:confusion}

In this section, we display the per-class predictions of various algorithms on the \mbox{sU $\rightarrow$ M} task. In Table~\ref{tab:conf_nann}, we see that without domain adaptation, performance on classes is rather random, the digit $9$ for instance is very poorly predicted and confused with $4$ and $8$.

Table~\ref{tab:conf_dann} shows an interesting pattern for $\dann$. In line with the limitations described by Theorem~\ref{thm:lower_bound}, the model performs very poorly on the subsampled classes (we recall that the subsampling is done in the source domain): the neural network tries to match the unweighted marginals. To do so, it projects representations of classes that are over-represented in the target domain (digits $0$ to $4$) on representations of the under-represented classes (digits $5$ to $9$). In doing so, it heavily degrades its performance on those classes (it is worth noting that digit $0$ has an importance weight close to $1$ which probably explains why $\dann$ still performs well on it, see Table~\ref{tab:weights}).

As far as $\iwdan$ is concerned, Table~\ref{tab:conf_iwdan} shows that the model perfoms rather well on all classes, at the exception of the digit $7$ confused with $9$. $\iwdano$ is shown in Table~\ref{tab:conf_iwdano} and as expected outperforms the other algorithms on all classes. 

Finally, Table~\ref{tab:weights} shows the estimated weights of all the algorithms, at the training epoch displayed in Tables~\ref{tab:conf_nann}, \ref{tab:conf_dann}, \ref{tab:conf_iwdan} and \ref{tab:conf_iwdano}. We see a rather strong correlation between errors on the estimated weight for a given class, and errors in the predictions for that class (see for instance digit $3$ for $\dann$ or digit $7$ for $\iwdan$).

\begin{table*}[ht]
\centering\small
	\caption{Estimated weights and their euclidian distance to the true weights, taken at the training epoch for the confusion matrices in Tables~\ref{tab:conf_nann}, \ref{tab:conf_dann}, \ref{tab:conf_iwdan} and \ref{tab:conf_iwdano}. The first row contains the true weights. The last column gives the euclidian distance from the true weights.}
	\label{tab:weights}
    \centering
    \begin{sc}
        \begin{tabular}{c|*{10}{c}|c}
            \toprule
            & \multicolumn{10}{c|}{Class} \\
            \midrule
            & 0 & 1 & 2 & 3 & 4 & 5 & 6 & 7 & 8 & 9 & Distance \\
            \midrule
            TRUE & 1.19 & 1.61 & 1.96 &  2.24 & 2.16 & 0.70 & 0.64 & 0.70 & 0.78 & 0.66 & 0 \\
            \midrule
            DANN & 1.06 & 1.15 & 1.66 & 1.33 & 1.95 & 0.86 & 0.72 & 0.70 & 1.02 & 0.92 & 1.15 \\
            IWDAN & 1.19 & 1.61 & 1.92 & 1.96 & 2.31 & 0.70 & 0.63 & 0.55 & 0.78 & 0.78 & 0.38 \\
            IWDAN-O & 1.19 & 1.60 & 2.01 & 2.14 & 2.1 & 0.73 & 0.64 & 0.65 & 0.78 & 0.66 & 0.12 \\
            No DA & 1.14 & 1.4 & 2.42 & 1.49 & 4.21 & 0.94 & 0.38 & 0.82 & 0.62 & 0.29 & 2.31 \\
            \bottomrule
        \end{tabular}
    \end{sc}
\end{table*}

\begin{table}[ht]
    \centering
    \caption{Ablation study on the Digits tasks, with weights learnt during training.}
    \label{tab:ablation_estimated_weights}
    \begingroup
    \begin{tabular}{lcclcc}
        \toprule
        Method & Digits & sDigits & Method & Digits & sDigits \\
        \midrule
        DANN & 93.15 & 83.24 & CDAN & 95.72 & 88.23  \\
        DANN + $\mathcal{L}_{C}^{\ww}$  & 93.18 & 84.20 &  CDAN + $\mathcal{L}_{C}^{\ww}$  & 95.30 & 91.14 \\
        DANN + $\mathcal{L}_{DA}^{\ww}$ & \textbf{94.35} & \textbf{92.48} & CDAN + $\mathcal{L}_{DA}^{\ww}$ & 95.42 & 92.35  \\
        IWDAN  & \textbf{94.90} & \textbf{92.54} & IWCDAN & \textbf{95.90} & \textbf{93.22}  \\        
        \bottomrule
    \end{tabular}
    \endgroup
\vspace*{-0.5em}
\end{table}


\renewcommand{\arraystretch}{0}
\setlength{\fboxsep}{3mm} 
\setlength{\tabcolsep}{0pt}

\begin{table}[ht]
	\caption{Per-class predictions without domain adaptation on the \mbox{sU $\rightarrow$ M} task. Average accuracy: $74.49\%$. The table $M$ below verifies $M_{ij} = \dist_T(\hat{Y} = j | Y = i)$.}
    \begin{center}
    	\begin{tabular}{*{10}{R}}
            92.89 &  0.13 &  3.24 &  0.00 &  2.20 &  0.01 &  0.45 &  0.88 &  0.18 &  0.02 \\
             0.00 & 72.54 & 12.38 &  0.00 &  3.40 &  0.37 &  7.54 &  1.50 &  2.28 &  0.00 \\
             0.31 &  0.23 & 93.28 &  0.09 &  0.72 &  0.03 &  0.34 &  4.78 &  0.17 &  0.05 \\
             0.06 &  0.77 &  4.81 & 68.53 &  1.50 & 19.91 &  0.02 &  2.48 &  1.61 &  0.31 \\
             0.02 &  0.62 &  0.28 &  0.00 & 97.19 &  0.51 &  0.04 &  0.17 &  0.79 &  0.37 \\
             0.75 &  3.03 &  0.69 &  1.01 &  1.20 & 88.96 &  0.39 &  0.31 &  2.69 &  0.96 \\
             0.73 &  1.98 &  0.42 &  0.03 & 23.86 &  2.74 & 69.08 &  0.29 &  0.12 &  0.75 \\
             1.02 &  2.01 &  4.16 &  0.13 &  9.32 &  6.36 &  0.01 & 73.48 &  1.01 &  2.50 \\
             6.01 &  8.27 &  2.55 &  1.35 &  1.62 &  3.62 &  4.98 &  6.96 & 64.40 &  0.24 \\
             1.49 &  3.35 &  0.55 &  1.28 & 38.30 & 15.36 &  0.05 & 20.68 &  1.34 & 17.60
        \end{tabular}
\label{tab:conf_nann}
	\end{center}
\end{table}

\begin{table}[ht]
	\caption{Per-class predictions for $\dann$ on the \mbox{sU $\rightarrow$ M} task. Average accuracy: $86.71\%$. The table $M$ below verifies $M_{ij} = \dist_T(\hat{Y} = j | Y = i)$. The first $5$ classes are under-represented in the source domain compared to the target domain. On those (except $0$), $\dann$ does not perform as well as on the over-represented classes (the last $5$). In line with Th.~\ref{thm:lower_bound}, matching the representation distributions on source and target forced the classifier to confuse the digits ``1'', ``3'' and ``4'' in the target domain with ``8'', ``5'' and ``9''.}
    \begin{center}
    	\begin{tabular}{*{10}{R}}
            95.79 &  0.01 &  0.08 &  0.01 &  0.12 &  0.38 &  2.34 &  0.36 &  0.36 &  0.57 \\
             0.14 & 70.77 &  0.80 &  0.01 &  1.03 &  1.29 &  9.46 &  0.06 & 16.39 &  0.06 \\
             1.61 &  0.14 & 89.82 &  0.20 &  0.42 &  0.48 &  0.83 &  3.73 &  1.37 &  1.39 \\
             0.46 &  0.08 &  1.10 & 63.33 &  0.04 & 26.28 &  0.02 &  1.78 &  3.76 &  3.14 \\
             0.11 &  0.13 &  0.05 &  0.00 & 78.17 &  0.85 &  0.16 &  0.15 &  1.97 & 18.41 \\
             0.19 &  0.04 &  0.02 &  0.04 &  0.01 & 91.30 &  0.25 &  0.27 &  5.97 &  1.91 \\
             0.62 &  0.12 &  0.01 &  0.00 &  1.98 &  4.61 & 91.73 &  0.05 &  0.36 &  0.51 \\
             0.14 &  0.23 &  1.39 &  0.13 &  0.10 &  0.32 &  0.02 & 94.10 &  1.46 &  2.10 \\
             0.69 &  0.13 &  0.11 &  0.05 &  0.21 &  2.12 &  0.50 &  0.36 & 95.19 &  0.66 \\
             0.36 &  0.31 &  0.03 &  0.08 &  0.46 &  3.67 &  0.01 &  1.64 &  1.03 & 92.40
        \end{tabular}
        \label{tab:conf_dann}
	\end{center}
\end{table}

\begin{table}[ht]
	\caption{Per-class predictions for IWDAN on the \mbox{sU $\rightarrow$ M} task. Average accuracy: $94.38\%$. The table $M$ below verifies $M_{ij} = \dist_T(\hat{Y} = j | Y = i)$.}
    \begin{center}
    	\begin{tabular}{*{10}{R}}
            97.33 &  0.06 &  0.23 &  0.01 &  0.20 &  0.43 &  1.29 &  0.19 &  0.18 &  0.09 \\
             0.00 & 97.71 &  0.41 &  0.05 &  0.56 &  0.69 &  0.14 &  0.03 &  0.34 &  0.07 \\
             0.70 &  0.16 & 96.32 &  0.08 &  0.34 &  0.23 &  0.23 &  1.49 &  0.43 &  0.01 \\
             0.23 &  0.01 &  0.97 & 87.67 &  0.01 &  9.25 &  0.02 &  0.63 &  0.87 &  0.35 \\
             0.11 &  0.25 &  0.05 &  0.00 & 96.93 &  0.16 &  0.22 &  0.02 &  0.40 &  1.85 \\
             0.15 &  0.11 &  0.01 &  0.16 &  0.05 & 95.81 &  0.69 &  0.11 &  2.82 &  0.08 \\
             0.26 &  0.25 &  0.00 &  0.00 &  2.07 &  1.49 & 95.84 &  0.01 &  0.07 &  0.00 \\
             0.16 &  0.42 &  2.12 &  0.91 &  1.15 &  0.60 &  0.03 & 82.07 &  1.35 & 11.19 \\
             0.44 &  0.50 &  0.36 &  0.18 &  0.43 &  0.90 &  0.91 &  0.15 & 95.74 &  0.40 \\
             0.34 &  0.42 &  0.06 &  0.30 &  1.67 &  2.46 &  0.11 &  0.31 &  0.85 & 93.50
        \end{tabular}
    	\label{tab:conf_iwdan}
	\end{center}
\end{table}

\begin{table}[ht]
	\caption{Per-class predictions for IWDAN-O on the \mbox{sU $\rightarrow$ M} task. Average accuracy: $96.8\%$. The table $M$ below verifies $M_{ij} = \dist_T(\hat{Y} = j | Y = i)$.}
    \begin{center}
    	\begin{tabular}{*{10}{R}}
            98.04 &  0.01 &  0.20 &  0.00 &  0.27 &  0.03 &  1.17 &  0.11 &  0.15 &  0.02 \\
             0.00 & 98.35 &  0.33 &  0.15 &  0.17 &  0.27 &  0.19 &  0.05 &  0.47 &  0.02 \\
             0.22 &  0.04 & 97.48 &  0.07 &  0.29 &  0.08 &  0.52 &  1.09 &  0.19 &  0.02 \\
             0.10 &  0.00 &  0.66 & 95.72 &  0.01 &  2.32 &  0.00 &  0.35 &  0.56 &  0.27 \\
             0.01 &  0.25 &  0.05 &  0.00 & 96.80 &  0.03 &  0.18 &  0.01 &  0.60 &  2.06 \\
             0.23 &  0.11 &  0.00 &  0.72 &  0.00 & 96.09 &  0.68 &  0.13 &  2.01 &  0.03 \\
             0.27 &  0.31 &  0.00 &  0.00 &  2.07 &  0.63 & 96.54 &  0.00 &  0.17 &  0.01 \\
             0.26 &  0.45 &  2.13 &  0.29 &  0.90 &  0.32 &  0.01 & 92.66 &  1.06 &  1.92 \\
             0.55 &  0.22 &  0.30 &  0.06 &  0.18 &  0.22 &  0.41 &  0.33 & 97.11 &  0.62 \\
             0.46 &  0.37 &  0.16 &  0.86 &  0.82 &  1.45 &  0.01 &  0.77 &  0.98 & 94.13
        \end{tabular}
    	\label{tab:conf_iwdano}
	\end{center}
\end{table}

\end{document}